\documentclass{article}

\PassOptionsToPackage{numbers}{natbib}

\usepackage[preprint]{neurips_2026}

\usepackage[utf8]{inputenc} % allow utf-8 input
\usepackage[T1]{fontenc}    % use 8-bit T1 fonts
\usepackage{hyperref}       % hyperlinks
\usepackage{url}            % simple URL typesetting
\usepackage{booktabs}       % professional-quality tables
\usepackage{amsfonts}       % blackboard math symbols
\usepackage{nicefrac}       % compact symbols for 1/2, etc.
\usepackage{microtype}      % microtypography
\usepackage{xcolor}         % colors
\usepackage{graphicx}
\usepackage{amsmath}
\usepackage{amsthm}
\usepackage{multirow}
\usepackage{booktabs}
\usepackage{algorithm}
\usepackage{algpseudocode}

\usepackage{tikz}
\usetikzlibrary{arrows.meta,positioning,calc}

\newtheorem{proposition}{Proposition}
\newtheorem{lemma}{Lemma}

\newtheorem{remark}{Remark}
\newtheorem{definition}{Definition}

\title{FlowADMM: Plug-and-play ADMM with Flow-based Renoise-Denoise Priors}

\author{%
  Hendrik Sommerhoff\\
  Computer Vision Group, University of Siegen\\
  Hölderlinstraße 3, 57076 Siegen, Germany\\
  \texttt{hendrik.sommerhoff@uni-siegen.de}
  \And
  Michael Moeller\\
  Computer Vision Group, University of Siegen\\
  Hölderlinstraße 3, 57076 Siegen, Germany\\
  \texttt{michael.moeller@uni-siegen.de}
}

\newcommand{\best}[1]{\textbf{#1}}
\newcommand{\second}[1]{\underline{#1}}
\DeclareMathOperator{\prox}{prox}
\DeclareMathOperator{\argmin}{arg min}

\begin{document}

\maketitle

\begin{abstract}
Plug-and-play (PnP) methods for solving inverse problems have recently achieved strong performance by leveraging denoising priors based on powerful generative diffusion and flow models. However, existing diffusion- and flow-based PnP methods typically rely on stochastic renoise-denoise operations, which complicate the analysis of their convergence behavior. In this work, we identify and formalize the deterministic renoise-denoise operator underlying flow-based plug-and-play methods. This perspective reveals that these methods implicitly define a deterministic operator given by the expectation of a denoiser over the latent noise distribution. Building on this insight, we propose \emph{FlowADMM}, a PnP algorithm that integrates the renoise-denoise operator into the classical alternating direction method of multiplier (ADMM) framework. We establish convergence guarantees for FlowADMM under weak Lipschitz conditions on the underlying flow network, and extend the analysis to non-stationary time schedules. Empirically, FlowADMM achieves state-of-the-art performance among flow-based PnP methods on a range of inverse problems, including denoising, deblurring, super-resolution, and inpainting, while requiring fewer data consistency evaluations than prior approaches. Code for FlowADMM is available at \url{https://github.com/hesom/FlowADMM}.
\end{abstract}

\section{Introduction}
Inverse problems such as denoising, deblurring, super-resolution, and inpainting are commonly formulated in the form
\begin{equation}
    y = Ax + \nu\quad\nu\sim\mathcal{N}(0, \sigma_y^2I_n),
\end{equation}
where the goal is to recover the true image $x\in\mathbb{R}^m$ from degraded measurements $y\in\mathbb{R}^n$ and a known (linear) forward operator $A: \mathbb{R}^m \to \mathbb{R}^n$.
These problems are often ill-posed.
While in recent years deep learning methods were very successful in approximating a direct mapping $\mathcal{G}(y) \approx x$, this approach has two problems: (i) Direct end-to-end approaches often have limited interpretability and theoretical guarantees required for safety-critical applications and (ii) training is costly and each inverse problem requires a separate network to be trained.

On the other hand classical variational approaches formulate this task as a minimization problem
\begin{equation}
    x^* = \argmin_x F_y(x) + R(x)
    \label{eq:variational}
\end{equation}
where $F_y$ is a data fidelity term, e.g.~$F_y(x) = \Vert Ax - y\rVert_2^2$, and $R$ is a regularization term encoding prior information about the solution. Common approaches to solve~\eqref{eq:variational} are first-order proximal algorithms~\cite{parikh2014proximal, boyd2010distributed}
such as forward-backward splitting (FBS) given by
\begin{equation}
    x_{k+1} = \prox_{\tau R} (x_k - \tau\nabla F_y(x_k)),
\end{equation}
where $\tau>0$ denotes the step size.

While choosing a data term $F_y$ is typically straight-forward, hand-crafted regularizers, such as the popular total variation (TV) regularization $R(x) = \Vert \nabla x\rVert_1$, often cannot perfectly capture the complex structure of the space of natural images.
Plug-and-Play (PnP) methods~\cite{heide2014flexisp, venkatakrishnan2013plug} combine the advantages of deep learning and variational approaches, and, for instance, replace the prox operator of hand-crafted regularizers with a (learned) denoiser $D$, yielding
\begin{equation}
    x_{k+1} = D(x_k - \tau\nabla F_y(x_k)).
\end{equation}
This substitution allows powerful pretrained denoisers to act as implicit image priors, enabling the same model to be applied across different inverse problems by changing only the forward operator $A$.

Early PnP methods rely on denoisers trained for additive Gaussian noise~\cite{romano2017little, meinhardt2017learning, rick2017one, zhang2017learning, hurault2021gradient}. More recent approaches instead leverage priors derived from generative models, such as GANs~\cite{duff2024regularising} and VAEs~\cite{gonzalez2019solving}.
Recently, generative diffusion models~\cite{daras2024survey, zhu2023denoising, graikos2022diffusion} and flow models~\cite{martin2025pnp, pourya2025flower, ben2024d} have received attention in the PnP literature. Diffusion and flow models provide a time-dependent family of denoisers $\{D_t\}_{t\in[0,1]}$ adapted to different noise levels, making them directly applicable to PnP.

A key challenge in integrating diffusion and flow denoisers into iterative reconstruction algorithms is that these models are trained to operate on samples distributed according to a specific noise level distribution $p_t$. During a reconstruction process, however, the iterates generally deviate from this distribution. Recent flow- and diffusion-based PnP methods therefore incorporate explicit renoising steps that perturb the current iterate toward the appropriate noise level before applying the denoiser, and potentially averaging this process over multiple noise samples (e.g.~\cite{martin2025pnp}).
Despite their strong empirical performance, existing renoise-denoise PnP methods are typically formulated as stochastic algorithms and are difficult to analyze theoretically.
In particular, the deterministic operator implicitly induced by averaging the denoiser output over the latent noise distribution has not been characterized, and as a consequence the convergence behavior of these methods remains poorly understood.

In this work, we identify the deterministic renoise-denoise operator underlying flow-based PnP methods and propose \emph{FlowADMM}, a PnP algorithm built on this operator. We derive convergence guarantees under weak Lipschitz conditions on the underlying flow network, including for non-stationary timestep schedules, and demonstrate strong performance on several inverse problems.

Our contributions are summarized as follows:
\begin{itemize}
    \item We identify and analyze the mean renoise-denoise operator underlying stochastic flow-based PnP methods.
    \item We propose FlowADMM, a PnP algorithm integrating the mean renoise-denoise operator into classical ADMM.
    \item We derive convergence criteria for FlowADMM with non-constant time schedules under weak Lipschitz assumptions on the underlying flow network.
    \item We propose a timestep-dependent Monte Carlo sampling strategy motivated by the coarse-to-fine behavior of the mean renoise-denoise operator.
    \item We demonstrate state-of-the-art performance on several inverse problems while requiring fewer data term evaluations than previous competitive stochastic averaging approaches.
\end{itemize}

\section{Preliminaries}
\subsection{Flow Matching}
Continuous generative flow models~\cite{chen2018neural} learn a transformation between a simple base distribution $p_0$ (e.g. standard Gaussian) and a target data distribution $p_1$. Given an initial sample $x_0 \sim p_0$, a sample $x_1 \sim p_1$ is obtained by following the flow $\phi_t(x_0)$ defined by the ordinary differential equation (ODE)
\begin{equation}
    \frac{\partial}{\partial t}\phi_t(x) = v_t^\theta(\phi_t(x)), \qquad \phi_0(x) = x
\end{equation}
where $v_t^\theta$ is a time-dependent velocity field parameterized by a neural network. We refer to $t \in [0,1]$ as "time" following the flow matching literature, although it does not represent physical time but rather parametrizes the interpolation between the noise distribution and the data distribution.

Flow matching~\cite{lipman2023flow} is a training technique that trains $v_t^\theta$ using regression to a target conditional velocity field given by a probability path between $p_0$ and $p_1$. A common choice, that we also assume throughout this paper, is linear interpolation, as e.g. used in rectified flow~\cite{liu2023flow} and optimal transport (OT) flow \cite{tong2024improving}, which defines a path conditioned on two endpoints ($x_0$ and $x_1$) by
\begin{equation}
    x_t = tx_1 + (1-t)x_0
\end{equation}
which defines intermediate distributions $p_t$ for $t\in[0, 1]$, resulting in the flow matching objective
\begin{equation}
    \min_\theta \mathbb{E}_{t \sim \mathcal{U}(0, 1), x_0 \sim p_0, x_1 \sim p_1}(\Vert v_t^\theta(x_t) - (x_1 - x_0)\rVert_2^2).
    \label{eq:flow-matching-objective}
\end{equation}
Under this construction, the trained velocity network induces a family of MMSE denoisers 
\begin{equation}
D_t(x_t) = \mathbb{E}[x_1\mid x_t] = x_t + (1-t)v_t^\theta(x_t).
\end{equation}

\subsection{Flow-based plug-and-play reconstruction}
Flow-based PnP methods apply the denoiser within an iterative reconstruction scheme while adapting the iterate to the corresponding noise level through an explicit renoising step. For example, PnP-Flow~\cite{martin2025pnp} performs the updates
\begin{align}
\label{eq:pnpFlow}
\begin{split}
    z_{k+1} &= x_{k} - \tau_k\nabla F_y(x_k)\\
    \tilde{z}_{k+1} &= t_kz_{k+1} + (1-t_k)\epsilon,\qquad \epsilon\sim p_0 = \mathcal{N}(0,1)\\
    x_{k+1} &= D_{t_k}(\tilde{z}_{k+1}),
    \end{split}
\end{align}
where $(t_k)_k$ is a time-step schedule controlling the noise level.
PnP-Flow additionally notes that the denoising step may be averaged over multiple noise realizations. However, this averaged mapping is introduced only as an algorithmic variant, and its intrinsic operator properties are not characterized. In particular, the existing convergence analysis is only carried out for the stochastic iteration and does not study the residual structure or Lipschitz behavior of the mean denoiser. In the next section, we make this averaged operator explicit and analyze it directly.

\section{FlowADMM}
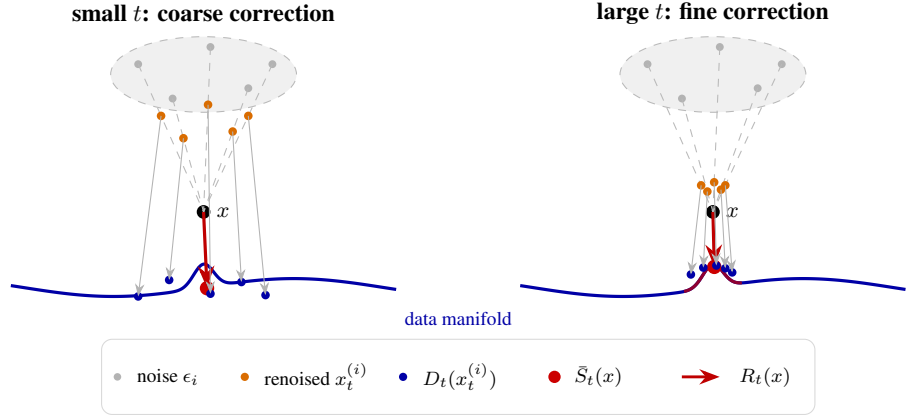
\begin{figure}[t]
\centering
\begin{tikzpicture}[
    >=Stealth,
    font=\small,
    scale=0.91, transform shape,
    noise/.style={circle, fill=gray!60, inner sep=1.15pt},
    xt/.style={circle, fill=orange!85!black, inner sep=1.25pt},
    clean/.style={circle, fill=blue!65!black, inner sep=1.25pt},
    xpoint/.style={circle, fill=black, inner sep=2pt},
    meanpoint/.style={circle, fill=red!80!black, inner sep=2.1pt},
    interp/.style={dashed, gray!55},
    denoise/.style={->, thin, gray!60},
    meanarrow/.style={->, very thick, red!75!black},
    manifold/.style={very thick, blue!65!black},
    fine/.style={very thick, purple!75!black},
    lab/.style={font=\footnotesize, align=center}
]

% ============================================================
% LEFT PANEL: SMALL t
% ============================================================
\begin{scope}[shift={(-3.7,0)}]

\node[font=\bfseries] at (0,2.95) {small $t$: coarse correction};

% Data manifold with a prominent local fine-detail bump
\draw[manifold, domain=-2.8:2.8, samples=240, smooth]
    plot (\x,{
        -1.05
        + 0.13*sin(1.2*\x r)
        + 0.34*exp(-24*\x*\x)
    });

% Current iterate
\coordinate (xL) at (0,0.05);
\node[xpoint] at (xL) {};
\node[right=2pt of xL] {$x$};

% Small t: coarse denoising outputs miss the fine bump
\coordinate (dL1) at (-0.95,-1.18);
\coordinate (dL2) at (-0.50,-0.94);
\coordinate (dL3) at (0.10,-1.14);
\coordinate (dL4) at (0.55,-0.97);
\coordinate (dL5) at (0.90,-1.16);

% Mean is between denoised samples, not on one of them
\coordinate (meanL) at (0.05,-1.06);

% Mean output and update
\node[meanpoint] at (meanL) {};
%\node[right=2pt of meanL] {$\bar S_t(x)$};
\draw[meanarrow] (xL) -- (meanL);

% Pure Gaussian noise distribution
\coordinate (cnoiseL) at (0,2.05);
\draw[fill=gray!12, draw=gray!55, dashed] (cnoiseL) ellipse (1.35 and 0.55);

% Noise samples
\coordinate (eL1) at (-0.95,2.20);
\coordinate (eL2) at (-0.45,1.70);
\coordinate (eL3) at (0.10,2.45);
\coordinate (eL4) at (0.65,1.85);
\coordinate (eL5) at (1.00,2.20);

\foreach \e in {eL1,eL2,eL3,eL4,eL5}{
    \node[noise] at (\e) {};
}

\coordinate (xtL1) at ($(eL1)!0.35!(xL)$);
\coordinate (xtL2) at ($(eL2)!0.35!(xL)$);
\coordinate (xtL3) at ($(eL3)!0.35!(xL)$);
\coordinate (xtL4) at ($(eL4)!0.35!(xL)$);
\coordinate (xtL5) at ($(eL5)!0.35!(xL)$);

\foreach \e/\xtt in {eL1/xtL1,eL2/xtL2,eL3/xtL3,eL4/xtL4,eL5/xtL5}{
    \draw[interp] (\e) -- (xL);
    \node[xt] at (\xtt) {};
}

\foreach \d in {dL1,dL2,dL3,dL4,dL5}{
    \node[clean] at (\d) {};
}

\foreach \xtt/\d in {xtL1/dL1,xtL2/dL2,xtL3/dL3,xtL4/dL4,xtL5/dL5}{
    \draw[denoise] (\xtt) -- (\d);
}

\end{scope}

\begin{scope}[shift={(3.7,0)}]

\node[font=\bfseries] at (0,2.95) {large $t$: fine correction};

\draw[manifold, domain=-2.8:2.8, samples=240, smooth]
    plot (\x,{
        -1.05
        + 0.13*sin(1.2*\x r)
        + 0.34*exp(-24*\x*\x)
    });

\draw[fine, domain=-0.42:0.42, samples=140, smooth]
    plot (\x,{
        -1.05
        + 0.13*sin(1.2*\x r)
        + 0.34*exp(-24*\x*\x)
    });

\coordinate (xR) at (0,0.05);
\node[xpoint] at (xR) {};
\node[right=2pt of xR] {$x$};

\coordinate (dR1) at (-0.32,-0.86);
\coordinate (dR2) at (-0.14,-0.76);
\coordinate (dR3) at (0.05,-0.73);
\coordinate (dR4) at (0.18,-0.77);
\coordinate (dR5) at (0.28,-0.83);

\coordinate (meanR) at (0.02,-0.75);

\node[meanpoint] at (meanR) {};
\draw[meanarrow] (xR) -- (meanR);

\coordinate (cnoiseR) at (0,2.05);
\draw[fill=gray!12, draw=gray!55, dashed] (cnoiseR) ellipse (1.35 and 0.55);

\coordinate (eR1) at (-0.95,2.20);
\coordinate (eR2) at (-0.45,1.70);
\coordinate (eR3) at (0.10,2.45);
\coordinate (eR4) at (0.65,1.85);
\coordinate (eR5) at (1.00,2.20);

\foreach \e in {eR1,eR2,eR3,eR4,eR5}{
    \node[noise] at (\e) {};
}

\coordinate (xtR1) at ($(eR1)!0.82!(xR)$);
\coordinate (xtR2) at ($(eR2)!0.82!(xR)$);
\coordinate (xtR3) at ($(eR3)!0.82!(xR)$);
\coordinate (xtR4) at ($(eR4)!0.82!(xR)$);
\coordinate (xtR5) at ($(eR5)!0.82!(xR)$);

\foreach \e/\xtt in {eR1/xtR1,eR2/xtR2,eR3/xtR3,eR4/xtR4,eR5/xtR5}{
    \draw[interp] (\e) -- (xR);
    \node[xt] at (\xtt) {};
}

\foreach \d in {dR1,dR2,dR3,dR4,dR5}{
    \node[clean] at (\d) {};
}

\foreach \xtt/\d in {xtR1/dR1,xtR2/dR2,xtR3/dR3,xtR4/dR4,xtR5/dR5}{
    \draw[denoise] (\xtt) -- (\d);
}

\end{scope}

\node[lab, blue!65!black] at (0,-1.50) {data manifold};

\begin{scope}[shift={(0,-2.35)}]
\node[draw=black!15, rounded corners, fill=white, inner sep=5pt] at (0,0) {
\begin{tikzpicture}[
    >=Stealth,
    font=\footnotesize,
    noise/.style={circle, fill=gray!60, inner sep=1.1pt},
    xt/.style={circle, fill=orange!85!black, inner sep=1.2pt},
    clean/.style={circle, fill=blue!65!black, inner sep=1.2pt},
    meanpoint/.style={circle, fill=red!80!black, inner sep=1.8pt},
    meanarrow/.style={->, very thick, red!75!black}
]
    \node[noise] at (0,0) {};
    \node[right=3pt] at (0,0) {noise $\epsilon_i$};

    \node[xt] at (1.85,0) {};
    \node[right=3pt] at (1.85,0) {renoised $x_t^{(i)}$};

    \node[clean] at (4.15,0) {};
    \node[right=3pt] at (4.15,0) {$D_t(x_t^{(i)})$};

    \node[meanpoint] at (6.35,0) {};
    \node[right=3pt] at (6.35,0) {$\bar S_t(x)$};

    \draw[meanarrow] (8.20,0) -- (8.75,0);
    \node[right=3pt] at (8.75,0) {$R_t(x)$};
\end{tikzpicture}
};
\end{scope}

\end{tikzpicture}
\caption{
Effect of the time parameter $t$ on the mean renoise-denoise operator.
For small $t$, interpolated points remain closer to noise and naturally spread out farther, producing diverse trajectories that recover coarse structure but average fine details.
For large $t$, interpolated points stay closer to $x$, producing trajectories that resolve the fine detail and place the mean operator on the local bump.
}
\label{fig:mean-renoise-denoise-t}
\end{figure}
\subsection{Mean renoise-denoise operator}
We now formalize the deterministic operator implicitly underlying sstochastic renoise-denoise plug-and-play methods.
Assume that we have a pretrained flow velocity network $v^\theta: [0,1] \times\mathbb{R}^m\rightarrow\mathbb{R}^m$ defining a time-dependent MMSE denoiser
\begin{equation}
    D_t = I + (1-t)v^\theta_t.
\end{equation}
Existing flow-based PnP methods apply the denoiser after explicitly renoising the current iterate with a random noise sample. Averaging over the latent noise distribution naturally defines the following deterministic operator, which we call the \emph{mean renoise-denoise operator}, as
\begin{equation}
    \bar{S}_t(x) = \mathbb{E}_{\epsilon\sim\mathcal{N}(0,I)}[D_t(tx + (1-t)\epsilon)]
\end{equation}
with timestep $t\in[0,1]$.
This operator represents the expected output of a renoise-denoise step and can be viewed as the deterministic limit of stochastic updates used by methods such as PnP-Flow~\cite{martin2025pnp} or Flower~\cite{pourya2025flower}. Figure~\ref{fig:mean-renoise-denoise-t} visualizes the action of the operator. The following remark shows that, under the ideal MMSE assumption, the operator leaves points on the data manifold unchanged, which implies that the data manifold is a fixed-point set of the operator.
\begin{remark}
    Assume that $D_t$ is the ideal MMSE denoiser for the flow path $x_t = tx_1 + (1-t)x_0$, where $x_0 \sim p_0$ and $x_1\sim p_1$, meaning that $D_t(x_t) = \mathbb{E}_{x_1}[x_1 \mid x_t]$.
    Then $\bar S_t(x) = x$ for $x~\sim p_1$.
    \label{remark:projection?}
\end{remark}

Our upcoming analysis involves the residual form of the mean renoise-denoise operator given by
\begin{equation}
    R_t(x) = \bar S_t(x) - x
\end{equation}
where $R_t$ represents the average correction induced by the denoiser at noise level $t$. The residual thus defines a vector field, which can be interpreted as a local correction toward the data manifold.

\subsection{Integration into ADMM}
We integrate the mean renoise-denoise operator into ADMM.
Compared to forward-backward splitting as used in \eqref{eq:pnpFlow}, ADMM allows multiple prior operators through auxiliary variables, can be extended to non-differentiable data terms, and does not contain an operator-dependend step size restriction~\cite{boyd2010distributed}.
The classical ADMM algorithm aims to solve \eqref{eq:variational} by restating it as the constrained problem
\begin{equation}
    \min_{x,z} F_y(x) + R(z)\quad \text{s.t.} \quad x = z,
\end{equation}
and solving it by minimizing the augmented Lagrangian
\begin{equation}
    \mathcal{L}_\rho(x, z, u) = F_y(x) + R(z) + \langle u, x - z \rangle + \frac{\rho}{2}\Vert x - z\rVert_2^2,
\end{equation}
where $\rho > 0$ is a penalty parameter and $u$ is the scaled dual variable.
The Lagrangian is minimized with alternating steps on $x$ and $z$. This leads to the following proximal updates
\begin{align}
    x_{k+1} &= \prox_{\tau F_y}(z_k - u_k)\\
    z_{k+1} &= \prox_{\tau R}(x_{k+1} + u_k)\\
    u_{k+1} &= u_k + x_{k+1} - z_{k+1}.
\end{align}
where $\tau = 1/\rho$ and the proximal operator is defined as
\begin{equation}
    \prox_f(v) = \argmin_x \frac{1}{2} \Vert x - v\rVert_2^2 + f(x).
\end{equation}
In the spirit of previous plug-and-play methods~\cite{ryu2019plug, meinhardt2017learning}, we propose to replace the proximal operator of the regularization function with the implicit prior given by the mean renoise-denoise operator, defining the algorithmic scheme, which we call FlowADMM
\begin{align}
    x_{k+1} &= \prox_{\tau F_y}(z_k - u_k)\\
    z_{k+1} &= \bar{S}_{t_k}(x_{k+1} + u_k)\\
    u_{k+1} &= u_k + x_{k+1} - z_{k+1},
\end{align}
where $(t_k)_k$ defines a time-schedule controlling the strength of the denoiser.

\subsection{Convergence analysis}
We now analyze the stability of the mean renoise-denoise operator, which will be used to develop convergence criteria for FlowADMM. 
Proofs in this section are given in Appendix~\ref{sec:proof-of-prop1}. 
\begin{lemma}
    \label{lemma:residual_lipschitz}
    Let the flow network $v_t^\theta$ be Lipschitz with constant $L_v(t)$.
    Then the residual operator $R_t = (\bar{S}_t - I)$ is Lipschitz continuous with constant $(1-t)(1+tL_v(t))$.
\end{lemma}
This result shows that the strength of the operator decreases as $t \to 1$, where $\bar S_t$ approaches the identity.
A convergence analysis of FlowADMM is delicate, since the mean denoise-renoise operators changes between iterations through the time schedule $(t_k)_k$.
Our analysis involves the notion of $\alpha$-averaged operators, which are defined as follows.
\begin{definition}
    A operator $T: \mathbb{R}^d \to \mathbb{R}^d$ is called, $\alpha$-averaged (or just averaged), if there exists $\alpha\in(0,1)$ and a non-expansive operator $N: \mathbb{R}^d \to \mathbb{R}^d$ such that
    \begin{equation}
        T = \alpha I + (1-\alpha) N,
    \end{equation}
    where $I$ is the identity operator.
\end{definition}
If an averaged operator has a fixed point, then the fixed point iteration $x_{k+1} = T(x_k)$ converges~\cite{monotone}. Moreover, $T$ has the same fixed point set as $N$.
We start the convergence analysis of FlowADMM with a proposition that gives convergence guarantees for a constant time schedule $t_k=t$, that is based on results of \citet{ryu2019plug}.
\begin{proposition}
    Let $t\in(0,1)$, let the flow network $v_t^\theta$ be Lipschitz continuous with constant $L_v(t) < \frac{1}{1-t}$ and assume $F_y$ is differentiable and $\mu$-strongly convex. Let $\xi = (1-t)(1+tL_v(t))$
    denote the Lipschitz constant of the residual operator $ R_t$.
    Then for data term penalty $\tau$ satisfying
    \begin{equation}
        \tau > \frac{\xi}{(1 + \xi - 2\xi^2)\mu},
    \end{equation}
    the operator $T_t$ representing the fixed-$t$ FlowADMM iteration, i.e. the operator mapping the FlowADMM state $(z_k, u_k)$ to $T_t(z_k, u_k) = (z_{k+1}, u_{k+1})$, is averaged. Thus, if a fixed point exists, fixed-t FlowADMM converges. 
    \label{proposition:convergence}
\end{proposition}
Proposition~\ref{proposition:convergence} establishes conditions on the underlying flow network and the step size of the data term in the simplified setting of a fixed timestep schedule. The Lipschitz condition $L_v < \frac{1}{1-t}$ is substantially weaker compared to standard contraction conditions in other PnP methods, which require the residual denoising networks to have Lipschitz constants less than 1. This is especially prominent if $t$ becomes close to 1.

In practice, PnP methods based on flow and diffusion models achieve better results when using increasing time schedules (or, equivalently, decreasing noise schedules in the diffusion setting).
For large $t$, the renoising distribution remains concentrated near the current iterate $x$. Consequently, the operator acts as a local refinement step only when $x$ is already close to the data distribution. If $x$ is far from the data manifold, large-$t$ denoising remains off-distribution and does not necessarily move the iterate toward a plausible image. This motivates using smaller values of $t$ in early iterations to obtain coarse global corrections, followed by larger values of $t$ for refinements once the iterate is closer to the learned prior. Note that Prop.~\ref{proposition:convergence} already covers the convergence analysis for time schedules that become constant after finitely many iterations, i.e., $t_k = t_{\max}$ for all $k \geq K$. The following proposition is generalization for non-constant time schedules.
\begin{proposition}
    Let $(t_k)_k$ be a sequence in $[0, t_\text{max}]$ with $\lim_{k\to \infty} t_k = t_\text{max} < 1$ and
    \begin{equation}
        \sum_{k=1}^\infty \vert t_k - t_{\text{max}} \rvert < \infty.
    \end{equation}
    Let $T_{t_k}$ denote the FlowADMM iteration at timestep $t_k$ and let there be a $\delta> 0$ such that for all $\hat t\in [t_\text{max} - \delta, t_\text{max}]$ the iterations $T_{\hat t}$ satisfy the assumptions in Proposition~\ref{proposition:convergence}. Let the underlying flow network additionally be Lipschitz in $t$, i.e. for some positive constant $L_t$
    \begin{equation}
        \Vert v_s^\theta(x) - v_t^\theta(x)\rVert \leq L_t\vert s - t\rvert \qquad \forall x.
    \end{equation}
    Then FlowADMM converges to a fixed point of $T_* = T_{t_\text{max}}$, if such a fixed point exists.
    \label{proposition:varying-t-admm-convergence}
\end{proposition}
This result can be interpreted as a convergence guarantee for a non-stationary fixed point iteration whose operators asymptotically approach a limiting averaged operator.
Notably, the Lipschitz restriction with respect to $x$ on the flow network only applies in the neighborhood of $t_\text{max}$, which is close to $1$ in practice. There is no contraction bound restriction on the Lipschitz constant with respect to time $L_t$.
%Practical schedules that fulfill the summability condition are, for instance, constant schedules, schedules that become constant after a finite number of iterations, and geometric schedules.

Although in this paper we do not explicitly enforce the Lipschitz condition on the flow networks, we give experimental evidence for the stability of FlowADMM by analyzing the Lipschitz behavior of pretrained networks during late-stage FlowADMM iterations in Appendix~\ref{sec:experiments-lipschitz}.

\subsection{Practical implementation}
In practice, the expectation in the definition of the mean renoise-denoise operator $\bar S_t$ is intractable and must be approximated using Monte Carlo sampling.
Given $N$ i.i.d.\ samples $\{\epsilon_i\}_{i=1}^N$, $\epsilon_i\sim \mathcal{N}(0,I)$, we define the empirical operator as the unbiased estimator
\begin{equation}
    \hat S_t^{(N)}(x) = \frac{1}{N}\sum_{i=1}^N D_t(tx + (1-t)\epsilon_i).
\end{equation}
In FlowADMM, we replace $\bar{S}_{t_k}$ with its Monte Carlo counterpart $\hat S_{t_k}^{(N_k)}$, yielding a stochastic iteration.
We control the number of samples and thus the variance of the operator \emph{per iteration} through the schedule $N_k$. Motivated by the coarse-to-fine behavior of $\bar S_t$ (see Figure~\ref{fig:mean-renoise-denoise-t}), we employ a sampling strategy in which the number of samples increases in later iterations.

For small $t_k$, and when the current iterate is still far away from the manifold, the operator makes large, coarse corrections and is relatively insensitive to sampling noise, so a small number of samples suffices. In contrast, during late iteration refinements, an accurate estimation becomes more critical to refine small local structure and not kick the iterate off the manifold. Increasing $N_k$ in later iterations therefore reduces variance when precision matters most.
In our implementation, we fix the total budget of flow evaluations and distribute it across iterations by increasing $N_k$ in late iterations.
This allocates more computational resources to later iterations, where higher accuracy is required.

% \paragraph{Over-relaxation.}
% {\color{red} This will most likely be removed}
% We furthermore investigate over-relaxation of the ADMM $z$-updates~\cite{boyd2010distributed, eckstein1992douglas, eckstein1994parallel, eckstein1998operator}. Over-relaxation is a common heuristic to improve convergence speed in ADMM.
% Specifically, we replace the argument of the $z$-update with an affine combination
% \begin{equation}
%     \tilde{x}_{k+1} = \alpha_\text{relax} x_{k+1} + (1-\alpha_\text{relax}) z_k,
% \end{equation}
% with relaxation parameter $\alpha_\text{relax} \in (0,2)$, and perform
% \begin{equation}
%     z_{k+1} = \hat S_{t_k}^{(N_k)}(\tilde{x}_{k+1} + u_k).
% \end{equation}
% In our experiments, we use $\alpha \geq 1$, which marginally improves results for some tasks.

\section{Related Work}
\paragraph{Stochastic regularization.}
SNORE~\cite{renaud2024plug} is closely related to our work in that it also addresses the distribution mismatch of applying denoisers to PnP iterates outside of their training noise levels. SNORE introduces an explicit stochastic regularization function based on Gaussian perturbations of the current iterate, leading to a stochastic gradient descent algorithm. It also introduces an annealed variant, in which the noise levels are gradually reduced.

Our work differs in three main aspects. First, SNORE is formulated as optimization of an explicit scalar regularizer, whereas our prior is defined directly through a deterministic mean renoise-denoise operator and does not, in general, correspond to the gradient of a scalar energy.
Second, SNORE relies on additive Gaussian corruption with Gaussian denoisers, while our construction is tailored to the linear interpolation path used by straight-line flow matching.
Third, SNORE analyzes the convergence of their stochastic gradient descent method to critical points under boundedness assumptions, whereas our analysis studies fixed-point convergence of a non-stationary ADMM operator under Lipschitz conditions on the underlying flow network.
%Moreover, SNORE performs stochastic updates using individual noise samples at each iteration, whereas FlowADMM explicitly approximates the deterministic mean renoise-denoise operator with a timestep-dependent Monte Carlo sampling strategy.

\paragraph{Flow-based PnP methods.}
Recent plug-and-play methods increasingly leverage denoisers derived from generative flow matching models~\cite{martin2025pnp, pourya2025flower, ben2024d}. PnP-Flow~\cite{martin2025pnp} introduces a renoise-denoise iteration within a forward-backward splitting framework and additionally observes that averaging over multiple noise realizations can improve stability. Their convergence analysis, however, relies on strong assumptions on the boundedness of iterates and vanishing data-term step-sizes.
Flower~\cite{pourya2025flower} further improves reconstruction quality by using proximal steps on the data term with adaptive step size. Their best-performing variant averages multiple complete stochastic reconstruction trajectories.

In contrast to these approaches, our work explicitly formalizes and analyzes the deterministic operator induced by averaging over the latent noise distribution and integrates it directly into an ADMM framework. Moreover, we derive convergence guarantees for non-constant timestep schedules with Lipschitz conditions directly on  the underlying flow network.

Beyond plug-and-play approaches, D-Flow~\cite{ben2024d} formulates inverse problems through optimization in the latent source space of the flow model by differentiating through the generation dynamics.

\paragraph{Convergence analysis of PnP methods.}
The convergence analysis of PnP algorithms in prior work commonly relies on contraction or averagedness properties of the denoising operator~\cite{ryu2019plug, chan2016plug, hurault2022proximal}. While these provide strong guarantees, directly extending them to flow-based methods is challenging due to the stochastic renoise-denoise setting and the time-dependence of the denoiser. Our analysis builds on classical PnP theory while accounting for flow-based renoise-denoise operators and non-stationary timestep schedules. Moreover, we do not require the denoiser to be non-expansive.

\paragraph{Posterior sampling methods for inverse problems.}
A related line of work uses diffusion or flow models for approximate posterior sampling in inverse problems~\cite{kawar2022denoising, chung2023diffusion, kim2025flowdps, erbach2026solving}. These approaches aim to sample from the posterior distribution conditioned on the measurements and are typically motivated from a Bayesian perspective. In contrast, plug-and-play methods are generally optimization-oriented and aim to compute a single high-quality reconstruction rather than sampling from the full posterior distribution. Our work follows this deterministic PnP perspective and studies the resulting reconstruction dynamics through an operator-theoretic analysis.

\section{Numerical experiments}
\subsection{Experimental setup and baselines}
We evaluate the performance of FlowADMM on a range of inverse problems and compare it to recent state-of-the-art plug-and-play methods based on flow and diffusion models. We consider tasks including denoising, deblurring, super-resolution, random inpainting and box inpainting. Details on each task is provided in Appendix~\ref{sec:task-details}.

Unless stated otherwise, we follow the experimental setup of Flower~\cite{pourya2025flower} to ensure a fair comparison. Specifically, we use 100 images of size $128\times 128$ from the CelebA dataset~\cite{liu2015faceattributes} and 100 images from the AFHQ-Cat dataset~\cite{choi2020starganv2}. For all experiments, we use the same pretrained flow models provided by~\cite{martin2025pnp}, which is also used in the experiments in~\cite{pourya2025flower}.

In addition to PnP-Flow~\cite{martin2025pnp} and Flower~\cite{pourya2025flower} we also compare against several other recent methods, including the flow-based approaches OT-ODE~\cite{pokle2023training}, D-Flow~\cite{ben2024d} and Flow-Priors~\cite{zhang2024flow}, the diffusion-based method DiffPIR~\cite{zhu2023denoising} and PnP-GS with a Gaussian gradient-step denoiser~\cite{hurault2021gradient}. For these methods, we report the quantitative results provided in~\cite{pourya2025flower}, which were obtained under the same experimental setup as our work.

\subsection{Computational budget}
To maintain a fair comparison, we match the total number of iterations $K$ and flow evaluations to those used by Flower5-OT~\cite{pourya2025flower}, which is our main competitor. Flower5-OT averages 5 independent runs of the stochastic Flower algorithm, which leads to $5K$ flow evaluations. We want to stress that although Flower5-OT uses the same amount of flow evaluations as FlowADMM, it actually requires $5\times$ more data-term proximal operator evaluations, because it averages full reconstructions, while FlowADMM only averages the prior step within each iteration. While flow evaluations often dominate the computational cost, the data consistency step can be non-negligible depending on the forward operator, making this reduction practically relevant.

Since the number of iterations is finite, we employ a power law time schedule between $t_\text{min}$ and $t_\text{max}$ defined as
\begin{equation}
    t_k = t_\text{min} + \left( \frac{k+1}{K}\right)^{\gamma}(t_\text{max} - t_\text{min}),
\end{equation}
where $\gamma < 1$ uses more late time steps and $\gamma > 1$ keeps the iterates near $t_\text{min}$ for longer.
For the Monte Carlo approximation of the mean renoise-denoise operator, we considered both a constant and a piecewise-constant three-stage sampling schedule, which allocates an increasing amount of samples to later iterations. We treat the positions and height of the steps as hyperparameters while keeping the total number of samples across all iterations fixed. This allows us to concentrate more computational resources on steps where higher accuracy is required.

\subsection{Hyperparameter tuning}
The hyperparameters of FlowADMM, including the data term weight $\tau$ and the step-based sampling schedule $(N_k)_k$, are tuned separately for each task on a validation set of 32 images per dataset, following the tuning protocol of~\cite{martin2025pnp, pourya2025flower}.
We observed that mainly the data term strength $\tau$ varies greatly between tasks, while the other parameters show a more consistent behavior.
We also found that a late-heavy three-phase $N_k$ schedule consistently outperforms the constant schedule overall, validating our sampling hypothesis. 
An ablation (see Appendix~\ref{sec:late-heavy-pnp-flow}) further shows that transferring the same late-heavy Monte carlo sampling schedule idea to PnP-Flow improves all five CelebA tasks significantly, but still remains below FlowADMM, indicating that both the schedule and the ADMM formulation contribute to
the final gains. Other ablations and further details on the tuning and the final per-task parameter choices can be found in Appendix~\ref{sec:hyperparameter-details}.

\subsection{Results}
\label{sec:results}

\begin{table}[ht]
\setlength{\tabcolsep}{2.5pt}
\centering
\caption{Results for different inverse problems on the CelebA test set.}
\label{tab:celeba-flower-admm}
\scriptsize
\begin{tabular}{lccccccccccccccc}
\toprule
Method &
\multicolumn{3}{c}{Denoising} &
\multicolumn{3}{c}{Deblurring} &
\multicolumn{3}{c}{Super-resolution} &
\multicolumn{3}{c}{Random inpainting} &
\multicolumn{3}{c}{Box inpainting} \\
\cmidrule(lr){2-4}\cmidrule(lr){5-7}\cmidrule(lr){8-10}\cmidrule(lr){11-13}\cmidrule(lr){14-16}
& PSNR & SSIM & LPIPS & PSNR & SSIM & LPIPS & PSNR & SSIM & LPIPS & PSNR & SSIM & LPIPS & PSNR & SSIM & LPIPS \\
\midrule
Degraded & 20.00 & 0.348 & 0.372 & 27.83 & 0.740 & 0.126 & 10.26 & 0.183 & 0.827 & 11.95 & 0.196 & 1.041 & 22.27 & 0.742 & 0.214 \\
PnP-GS & 32.64 & 0.910 & 0.035 & 34.03 & 0.924 & 0.041 & 31.31 & 0.892 & 0.064 & 29.22 & 0.875 & 0.070 & -- & -- & -- \\
DiffPIR & 31.20 & 0.885 & 0.060 & 32.77 & 0.912 & 0.060 & 31.52 & 0.895 & 0.033 & 31.74 & 0.917 & 0.025 & -- & -- & -- \\
OT-ODE & 30.54 & 0.859 & \second{0.032} & 33.01 & 0.921 & \second{0.029} & 31.46 & 0.907 & \best{0.025} & 28.68 & 0.871 & 0.051 & 29.40 & 0.920 & 0.038 \\
D-Flow & 26.04 & 0.607 & 0.092 & 31.25 & 0.854 & 0.038 & 30.47 & 0.843 & 0.026 & 33.67 & 0.943 & \best{0.015} & 30.70 & 0.899 & 0.026 \\
Flow-Priors & 29.34 & 0.768 & 0.134 & 31.54 & 0.858 & 0.056 & 28.35 & 0.713 & 0.102 & 32.88 & 0.871 & 0.019 & 30.07 & 0.858 & 0.048 \\
PnP-Flow1 & 31.80 & 0.905 & 0.044 & 34.48 & 0.936 & 0.040 & 31.09 & 0.902 & 0.045 & 33.05 & 0.944 & \second{0.018} & 30.47 & 0.933 & 0.037 \\
Flower1-OT & 32.28 & 0.914 & 0.034 & 34.98 & 0.947 & \best{0.026} & 32.36 & 0.923 & 0.034 & 33.08 & 0.944 & \second{0.018} & 31.19 & 0.945 & \best{0.022} \\
PnP-Flow5 & 32.30 & 0.911 & 0.056 & 34.80 & 0.940 & 0.047 & 31.49 & 0.906 & 0.056 & \second{33.98} & \second{0.953} & 0.022 & 31.09 & 0.940 & 0.043 \\
Flower5-OT & \second{33.14} & \best{0.926} & 0.038 & \best{35.67} & \best{0.954} & 0.032 & \second{33.09} & \second{0.932} & 0.040 & 33.95 & \second{0.953} & 0.020 & \best{31.87} & \best{0.952} & \second{0.023} \\
\midrule
FlowADMM (ours)& \best{33.19} & \second{0.918} & \best{0.027} & \second{35.53} & \second{0.950} & 0.035 & \best{33.94} & \best{0.940} & \second{0.025} & \best{35.23} & \best{0.965} & 0.020 & \second{31.72} & \second{0.949} & 0.034 \\
\bottomrule
\end{tabular}
\end{table}

\begin{table}[ht]
\setlength{\tabcolsep}{2.5pt}
\centering
\caption{Results for different inverse problems on the AFHQ-Cat test set.}
\label{tab:afhq-flower-admm}
\scriptsize
\begin{tabular}{lccccccccccccccc}
\toprule
Method &
\multicolumn{3}{c}{Denoising} &
\multicolumn{3}{c}{Deblurring} &
\multicolumn{3}{c}{Super-resolution} &
\multicolumn{3}{c}{Random inpainting} &
\multicolumn{3}{c}{Box inpainting} \\
\cmidrule(lr){2-4}\cmidrule(lr){5-7}\cmidrule(lr){8-10}\cmidrule(lr){11-13}\cmidrule(lr){14-16}
& PSNR & SSIM & LPIPS & PSNR & SSIM & LPIPS & PSNR & SSIM & LPIPS & PSNR & SSIM & LPIPS & PSNR & SSIM & LPIPS \\
\midrule
Degraded & 20.00 & 0.314 & 0.509 & 23.94 & 0.517 & 0.444 & 11.70 & 0.208 & 0.873 & 13.36 & 0.223 & 1.081 & 21.80 & 0.740 & 0.198 \\
PnP-GS & \second{32.58} & \best{0.894} & \best{0.072} & 27.91 & 0.753 & 0.349 & 24.15 & 0.632 & 0.362 & 29.42 & 0.836 & 0.126 & -- & -- & -- \\
DiffPIR & 30.58 & 0.835 & 0.189 & 27.56 & 0.728 & 0.342 & 23.65 & 0.624 & 0.402 & 31.70 & 0.881 & 0.062 & -- & -- & -- \\
OT-ODE & 30.03 & 0.815 & \second{0.076} & 27.06 & 0.713 & \best{0.123} & 25.91 & 0.716 & \best{0.108} & 29.40 & 0.839 & 0.090 & 24.62 & 0.875 & 0.085 \\
D-Flow & 26.13 & 0.574 & 0.175 & 27.82 & 0.721 & \second{0.164} & 24.64 & 0.601 & 0.190 & 32.20 & 0.894 & \second{0.040} & 26.26 & 0.842 & 0.077 \\
Flow-Priors & 29.41 & 0.763 & 0.153 & 26.47 & 0.700 & 0.181 & 23.51 & 0.570 & 0.272 & 32.37 & 0.906 & 0.047 & 26.20 & 0.818 & 0.118 \\
PnP-Flow1 & 31.18 & 0.863 & 0.135 & 27.87 & 0.760 & 0.304 & 26.94 & \second{0.763} & \second{0.171} & 33.00 & 0.918 & \best{0.037} & 26.00 & 0.897 & 0.103 \\
Flower1-OT & 31.69 & 0.879 & 0.102 & 28.64 & 0.775 & 0.255 & 26.23 & 0.741 & 0.272 & 32.97 & 0.918 & \second{0.040} & 26.19 & \second{0.915} & \best{0.063} \\
PnP-Flow5 & 31.43 & 0.864 & 0.168 & 28.19 & 0.766 & 0.332 & \second{27.37} & \second{0.774} & 0.183 & \second{33.75} & \second{0.929} & 0.048 & 26.68 & 0.901 & 0.120 \\
Flower5-OT & 32.35 & \second{0.891} & 0.116 & \best{28.97} & \second{0.784} & 0.283 & 26.57 & 0.750 & 0.282 & 33.70 & 0.927 & 0.045 & \second{26.88} & \best{0.922} & \second{0.066} \\
\midrule
FlowADMM (ours) & \best{32.90} & 0.890 & 0.082 & \second{28.96} & \best{0.792} & 0.284 & \best{28.12} & \best{0.808} & 0.199 & \best{34.82} & \best{0.938} & 0.048 & \best{27.39} & 0.903 & 0.068 \\
\bottomrule
\end{tabular}
\end{table}

The quantitative evaluation in Table~\ref{tab:celeba-flower-admm} and Table~\ref{tab:afhq-flower-admm} show that FlowADMM matches or improves the best competing methods in distortion metrics. In particular, FlowADMM obtains best PSNR in 7 out of 10 tasks and the best SSIM on 6 out of 10 tasks, with especially large gains on for super-resolution and inpainting. FlowADMM performs worse on the LPIPS metric, which is consistent with the perception-distortion tradeoff~\cite{blau2018perception}, since we optimized for PSNR. Qualitative examples in Figure~\ref{fig:result-images} further show that FlowADMM preserves fine structures, such as the cat whiskers in the super-resolution example. We give additional results in the Appendix.

\newcommand{\imgw}{0.12\textwidth}
  \newcommand{\psnrcell}[2]{%
    \shortstack[c]{\scriptsize #1\\[-2pt]\includegraphics[width=\imgw]{#2}}%
  }
  \newcommand{\plainimg}[1]{%
    \shortstack[c]{\strut\\[-2pt]\includegraphics[width=\imgw]{#1}}%
  }
  \newcommand{\vrow}[2]{%
    \rotatebox{90}{\parbox{1.65cm}{\centering #1\\#2}}%
  }

  \begin{figure}[t]
  \centering
  \setlength{\tabcolsep}{1pt}
  \renewcommand{\arraystretch}{0.0}
  \scriptsize
  \begin{tabular}{cccccccc}
  \toprule
   & GT & Degraded & OT-ODE & Flow-Priors & PnP-Flow5 & Flower5-OT & FlowADMM \\
  \midrule
  \vrow{CelebA}{Denoising} &
  \plainimg{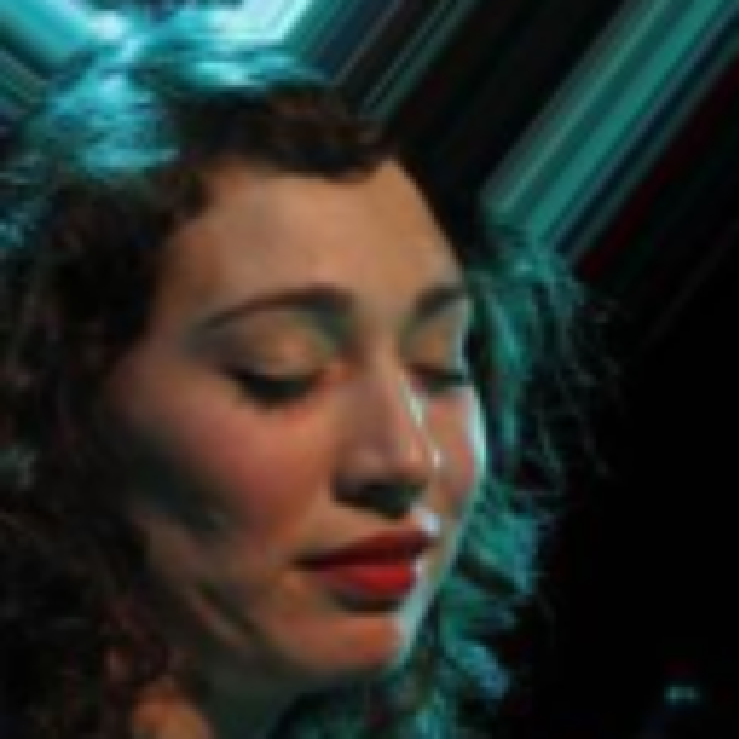} &
  \psnrcell{20.01}{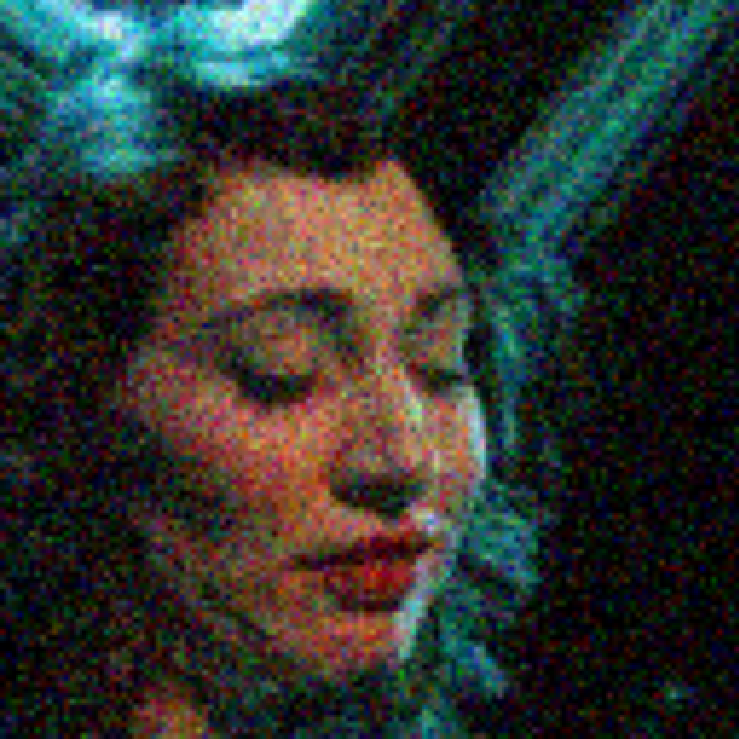} &
  \psnrcell{29.48}{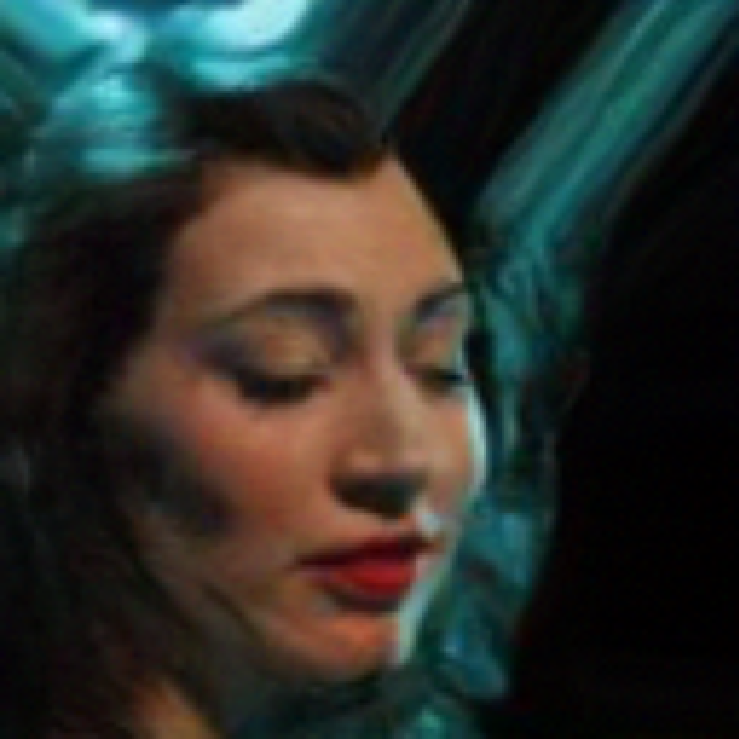} &
  \psnrcell{29.63}{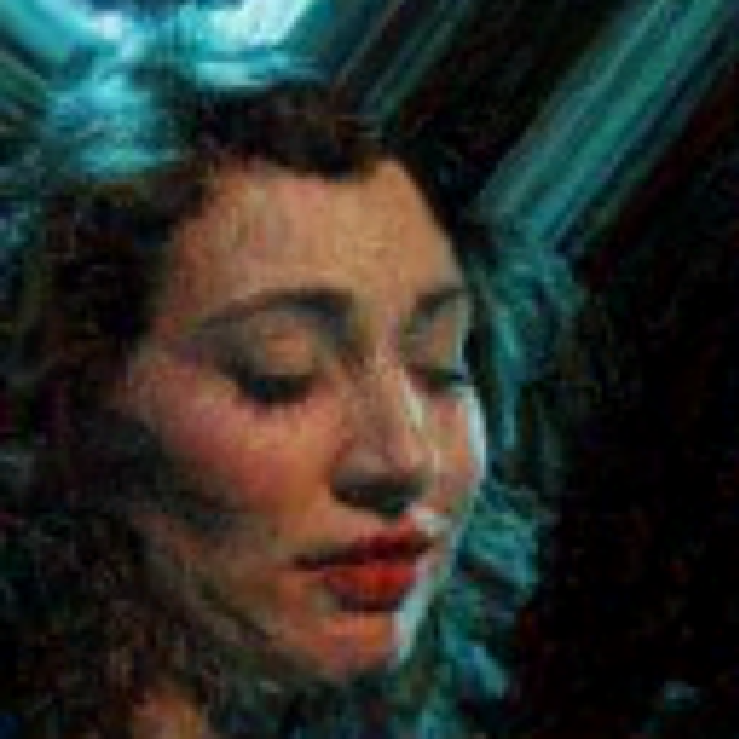} &
  \psnrcell{31.38}{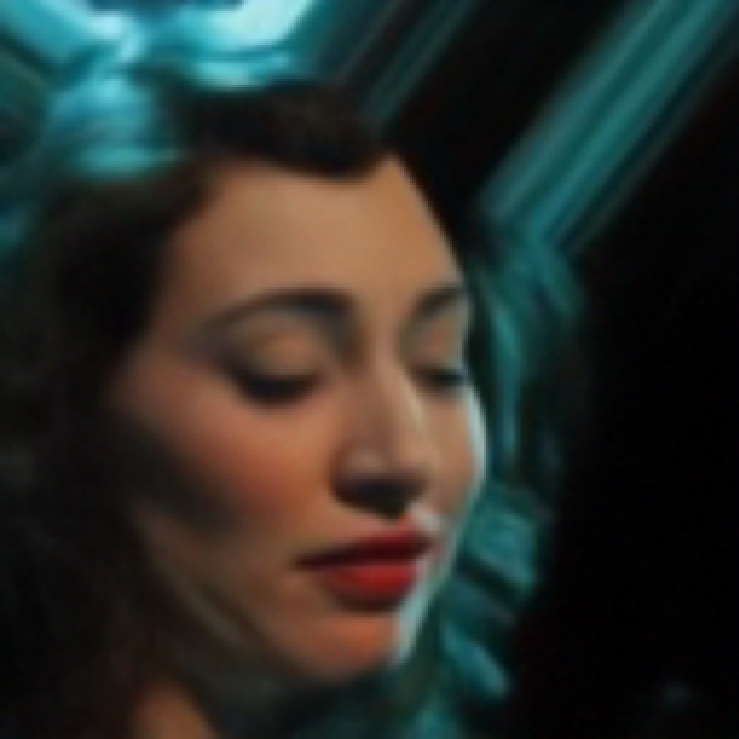} &
  \psnrcell{32.22}{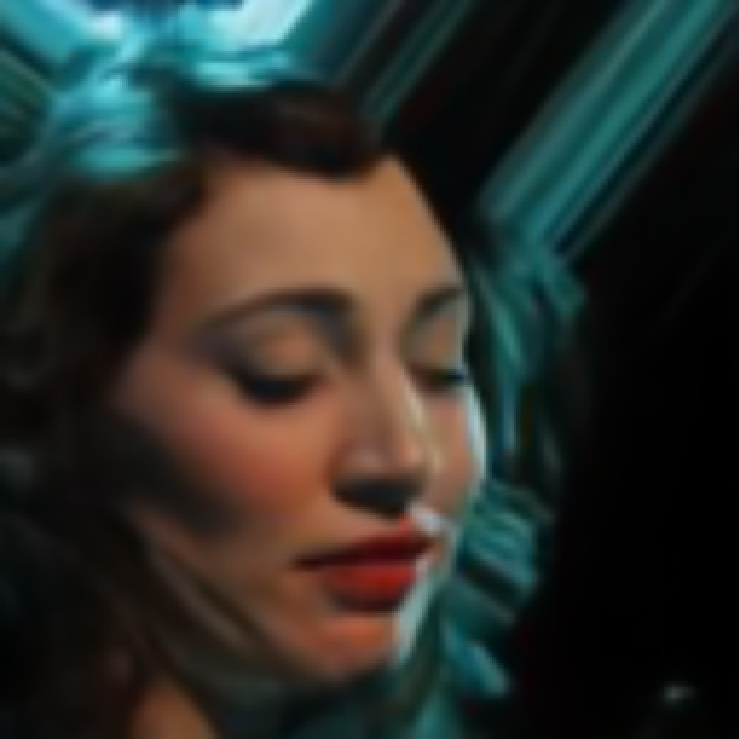} &
  \psnrcell{32.35}{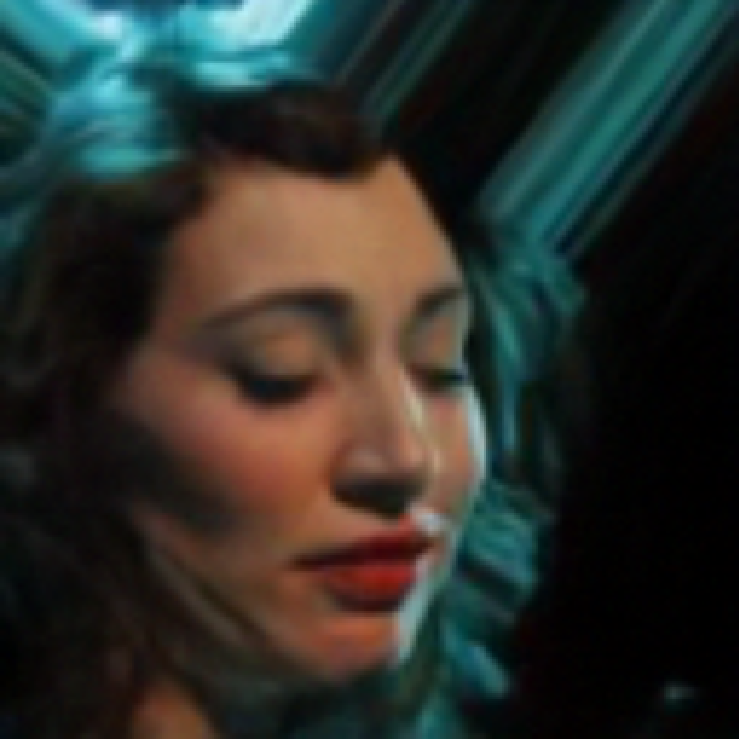} \\
  \vrow{CelebA}{Deblurring} &
  \plainimg{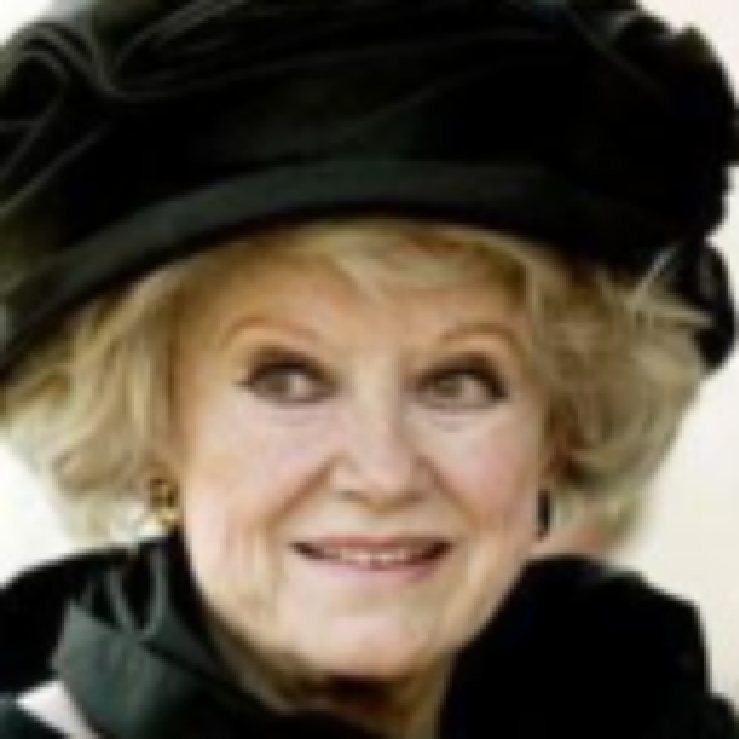} &
  \psnrcell{26.04}{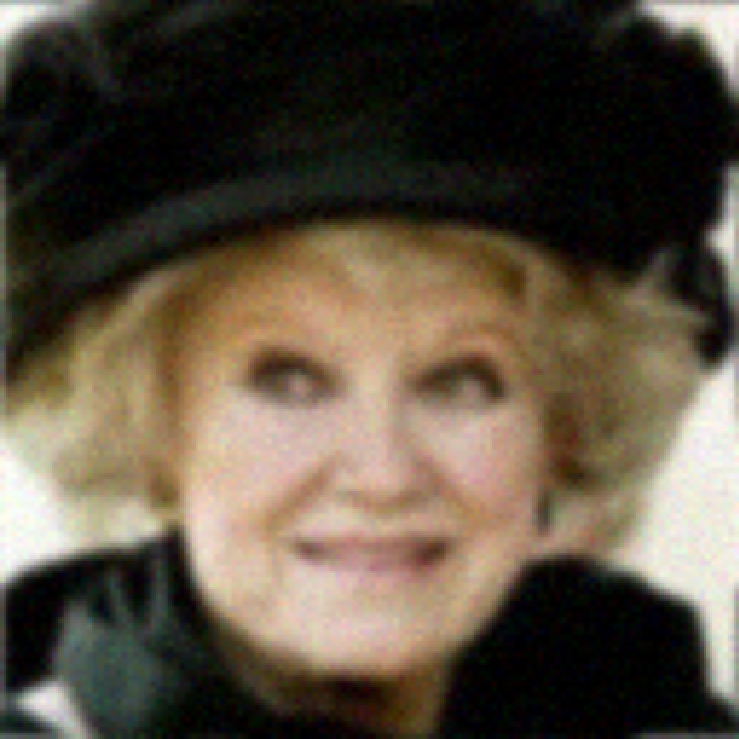} &
  \psnrcell{32.09}{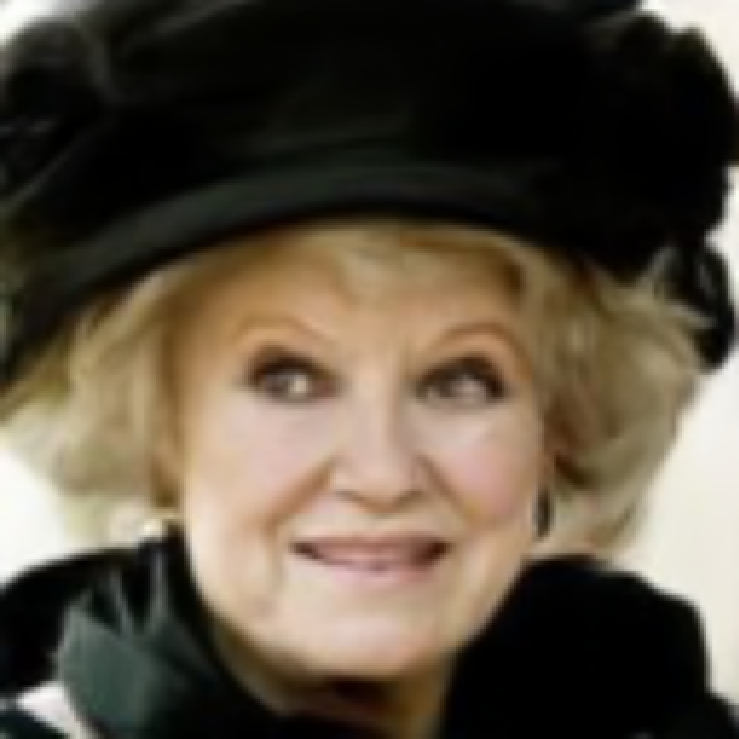} &
  \psnrcell{30.95}{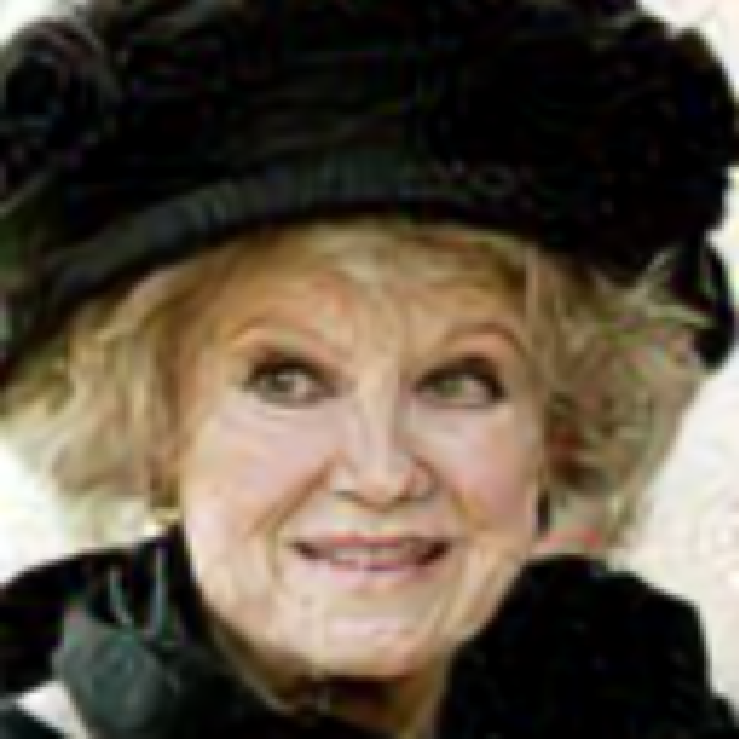} &
  \psnrcell{31.15}{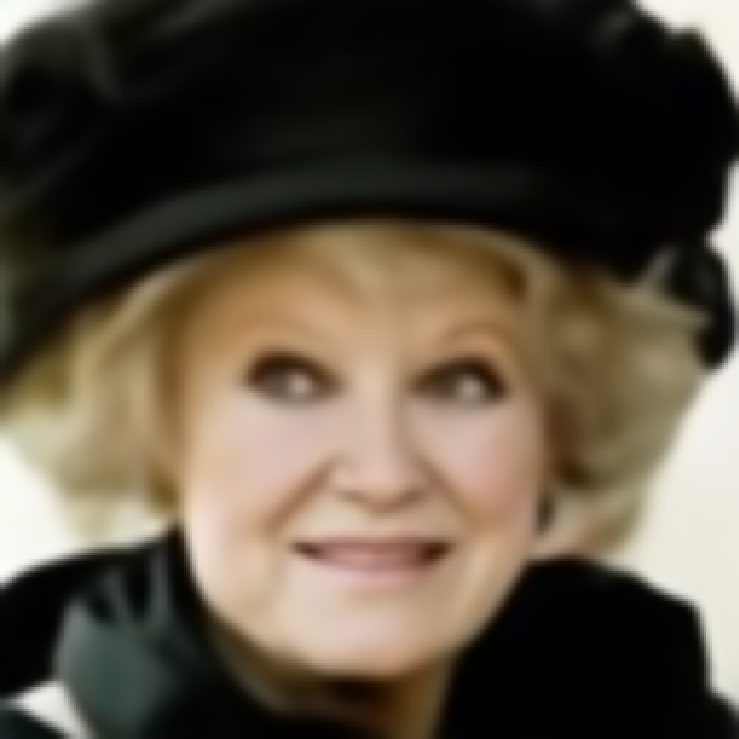} &
  \psnrcell{34.82}{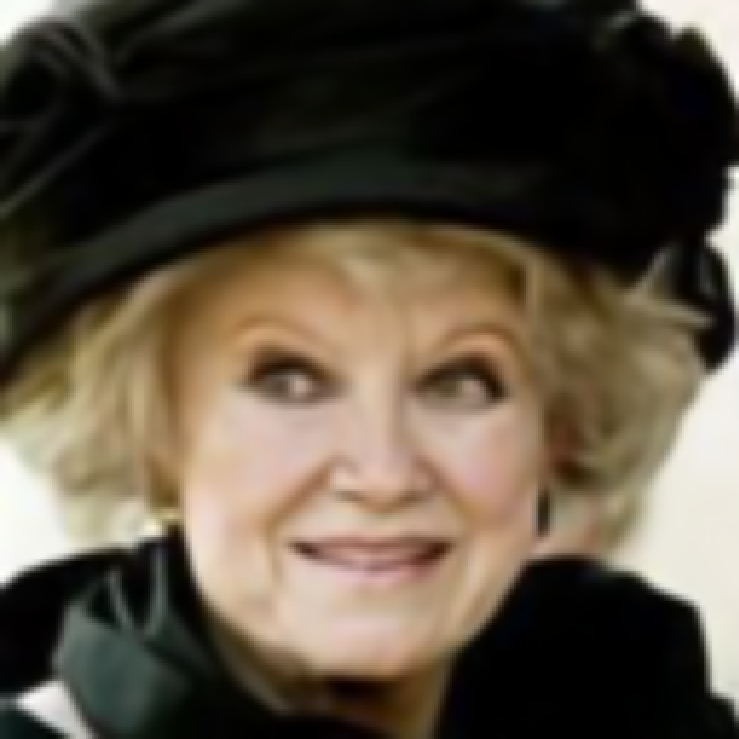} &
  \psnrcell{34.82}{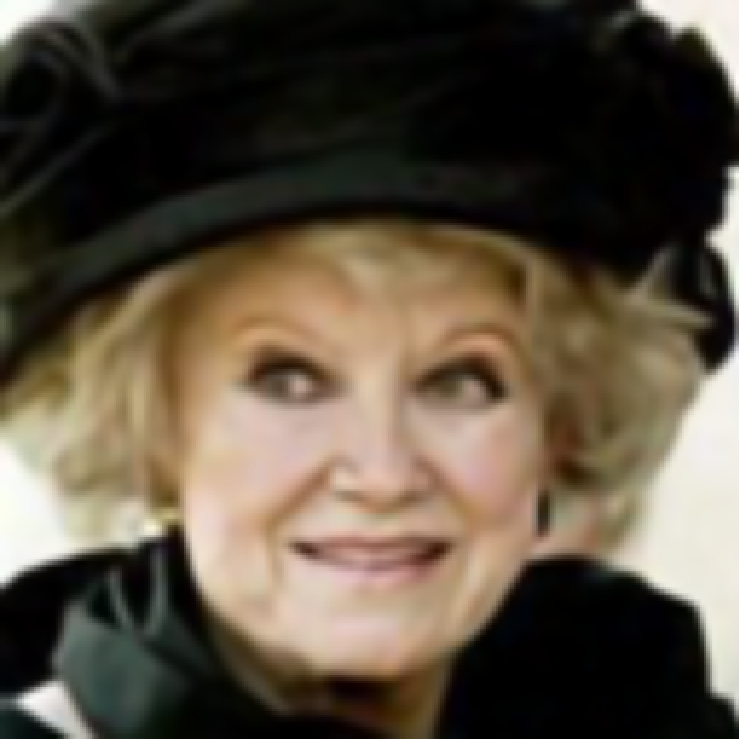} \\
  \vrow{AFHQ-Cat}{Super-Resolution} &
  \plainimg{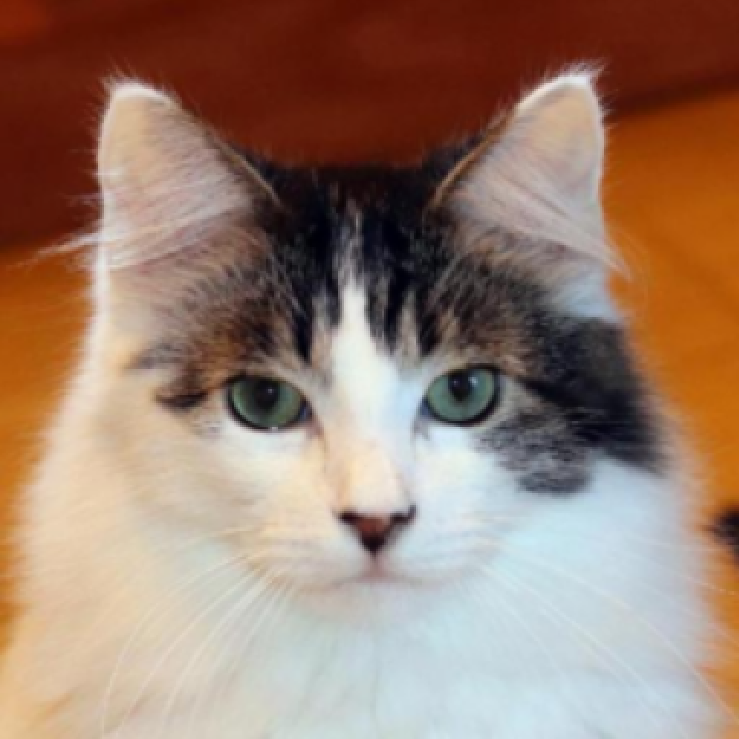} &
  \psnrcell{9.75}{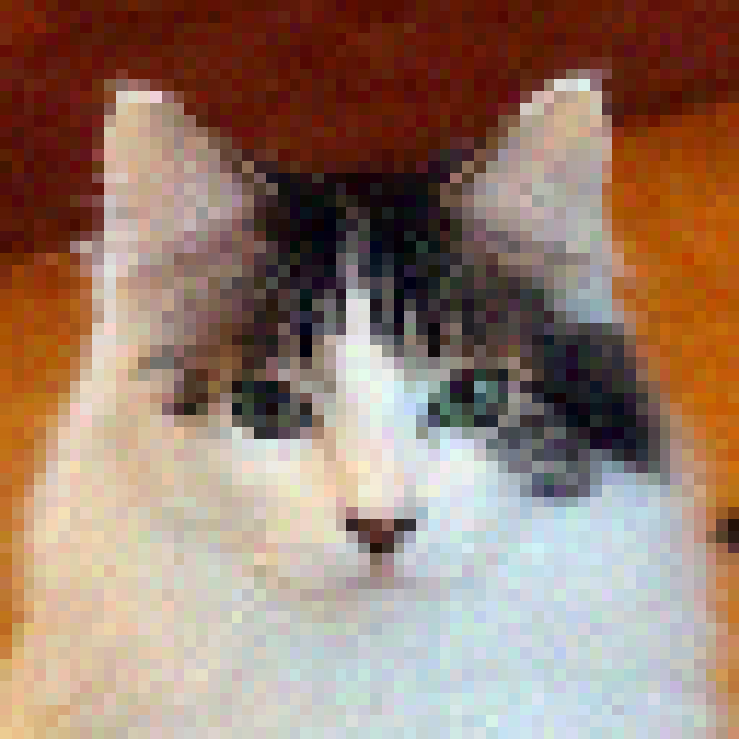} &
  \psnrcell{28.44}{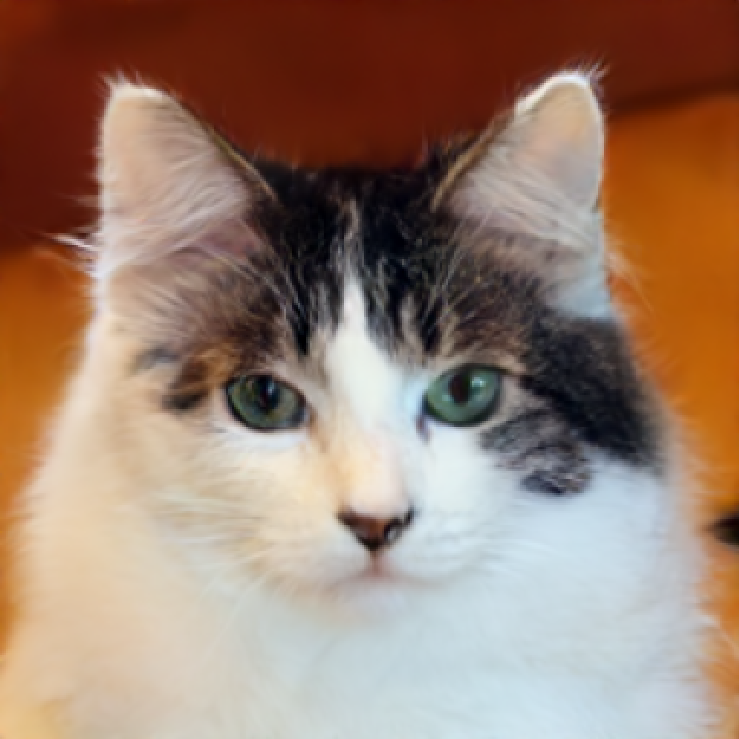} &
  \psnrcell{26.06}{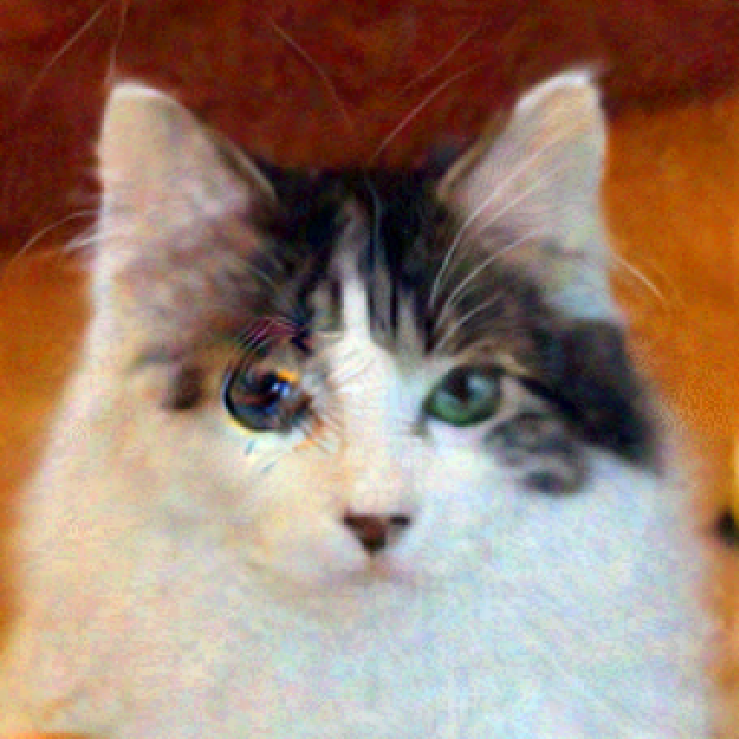} &
  \psnrcell{29.66}{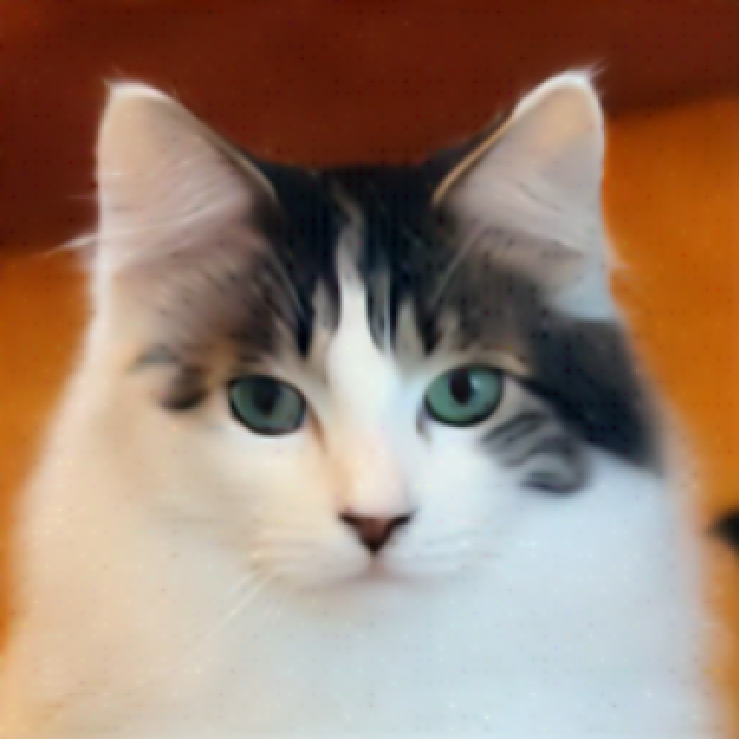} &
  \psnrcell{28.08}{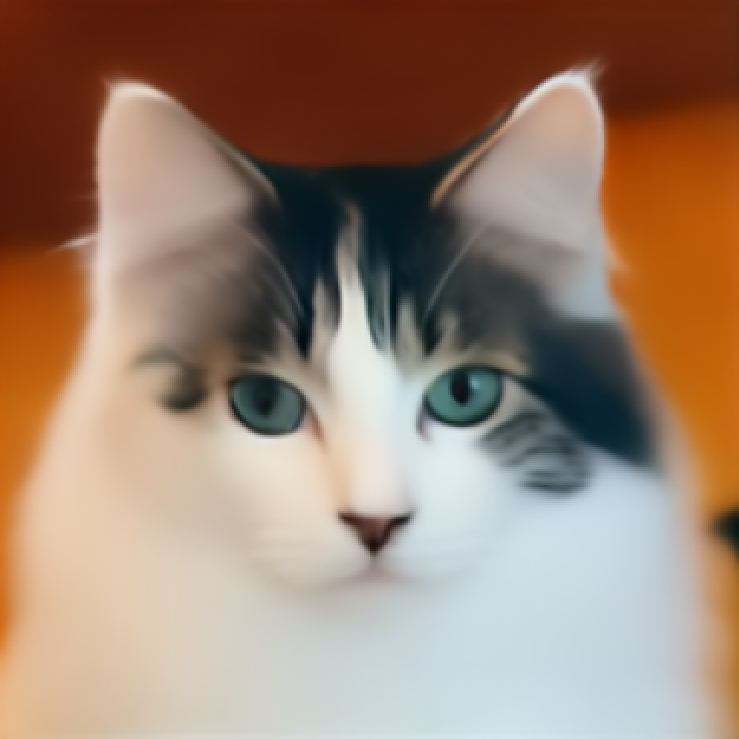} &
  \psnrcell{30.24}{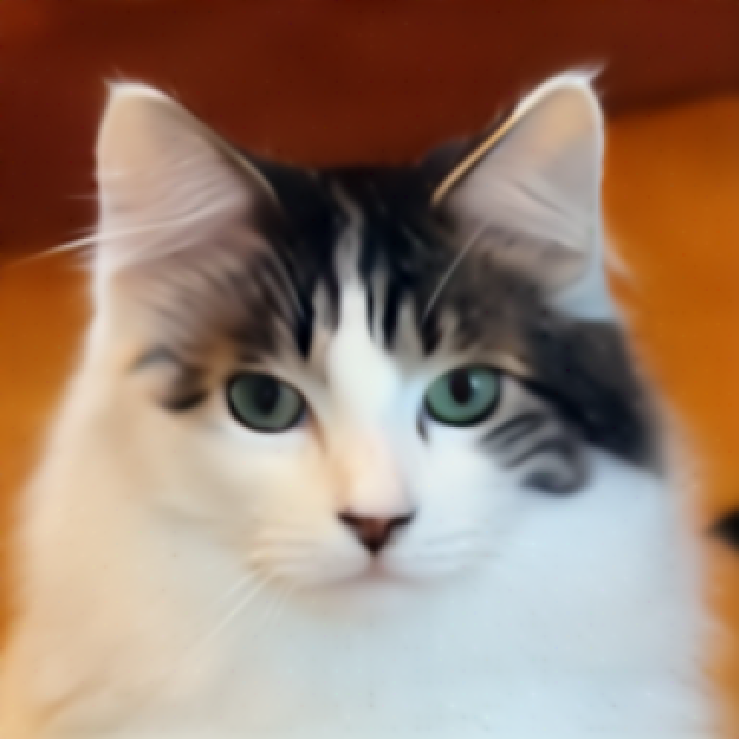} \\
  \vrow{AFHQ-Cat}{Box Inpainting} &
  \plainimg{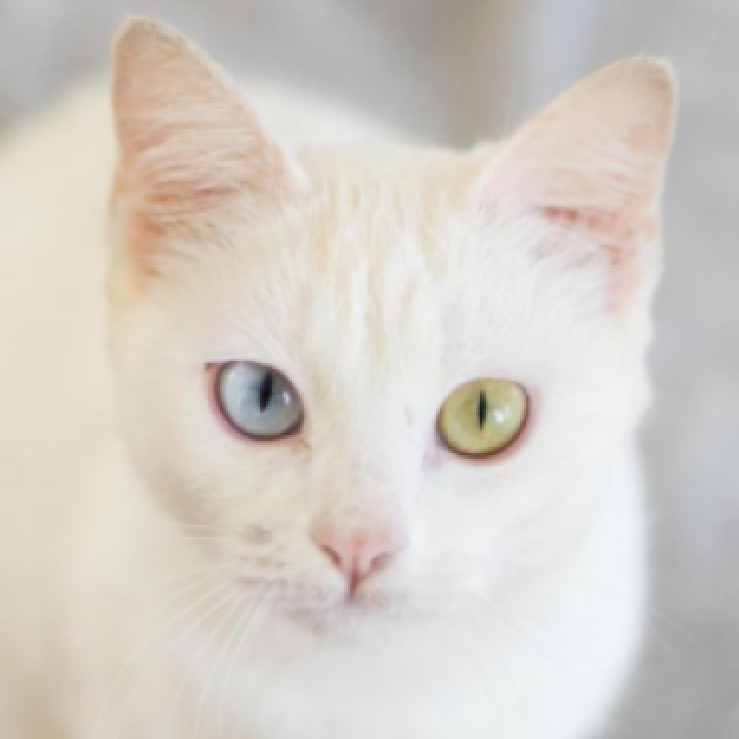} &
  \psnrcell{18.83}{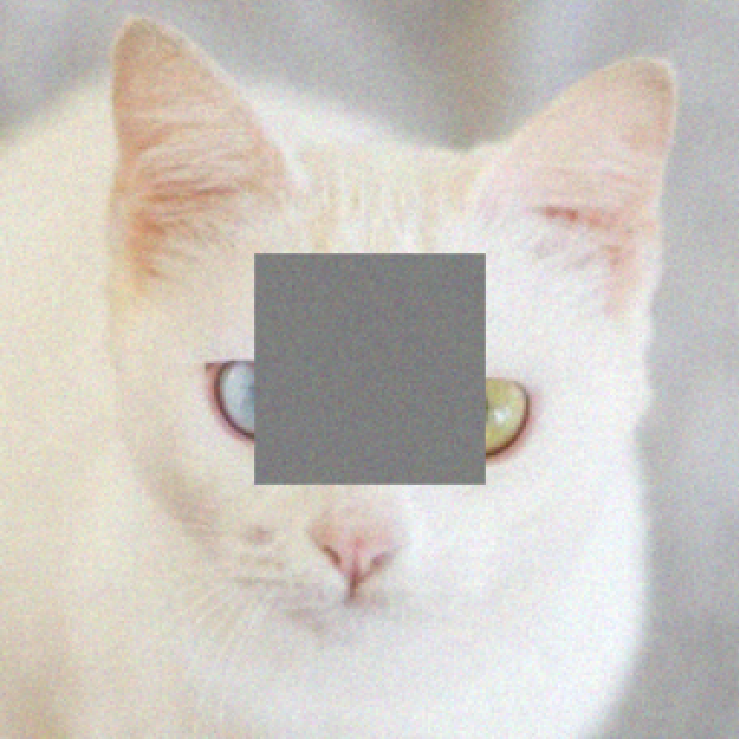} &
  \psnrcell{20.54}{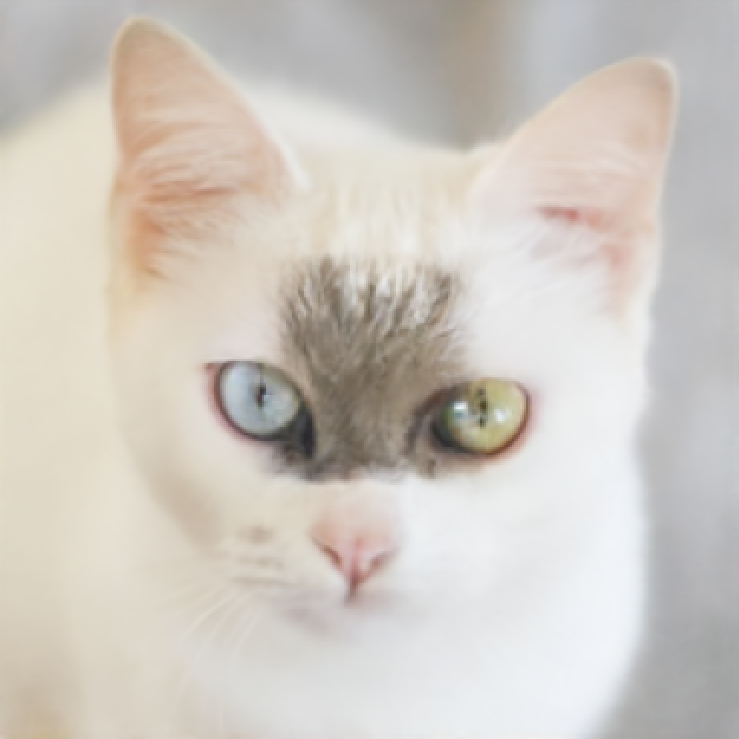} &
  \psnrcell{28.72}{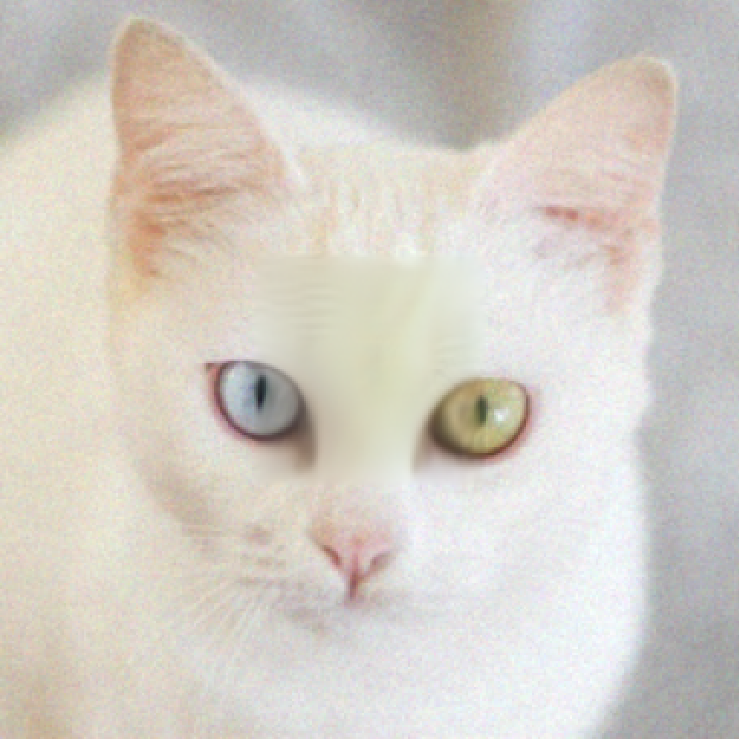} &
  \psnrcell{27.28}{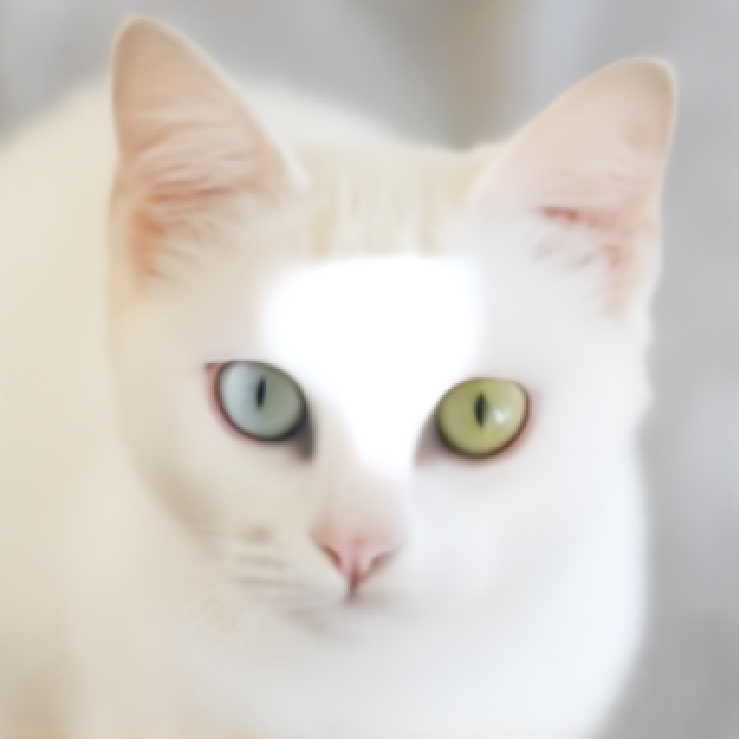} &
  \psnrcell{27.47}{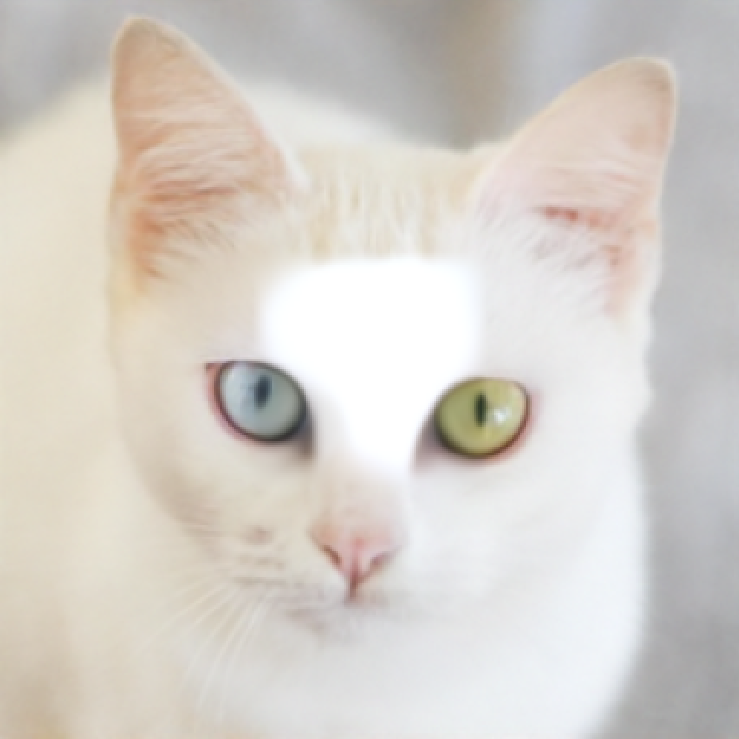} &
  \psnrcell{28.13}{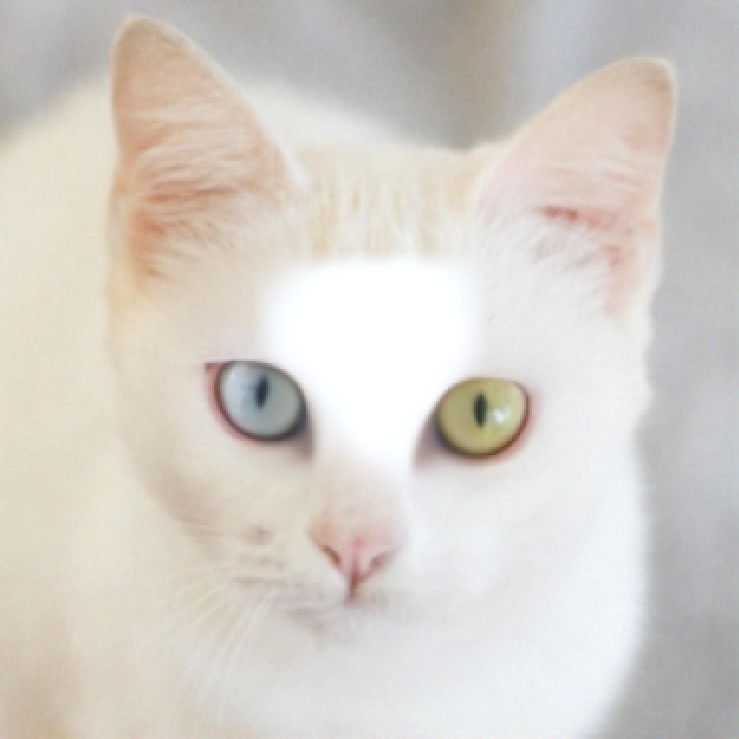} \\
  \vrow{CelebA}{Random Inpainting} &
  \plainimg{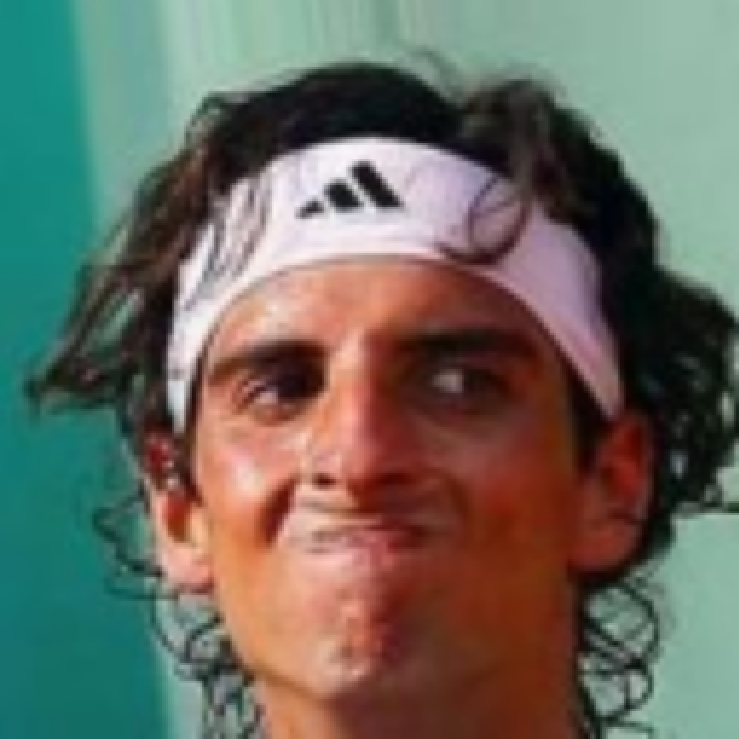} &
  \psnrcell{12.92}{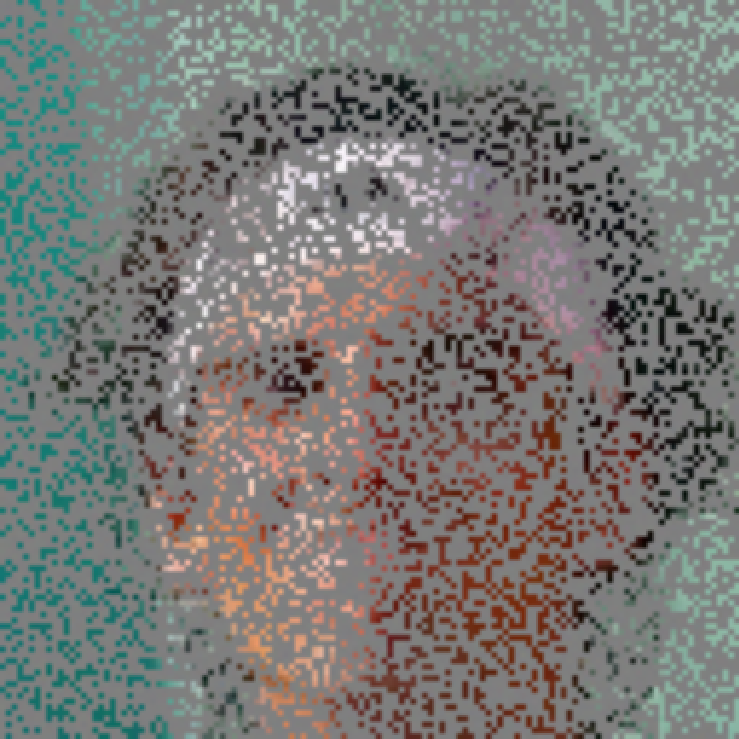} &
  \psnrcell{24.95}{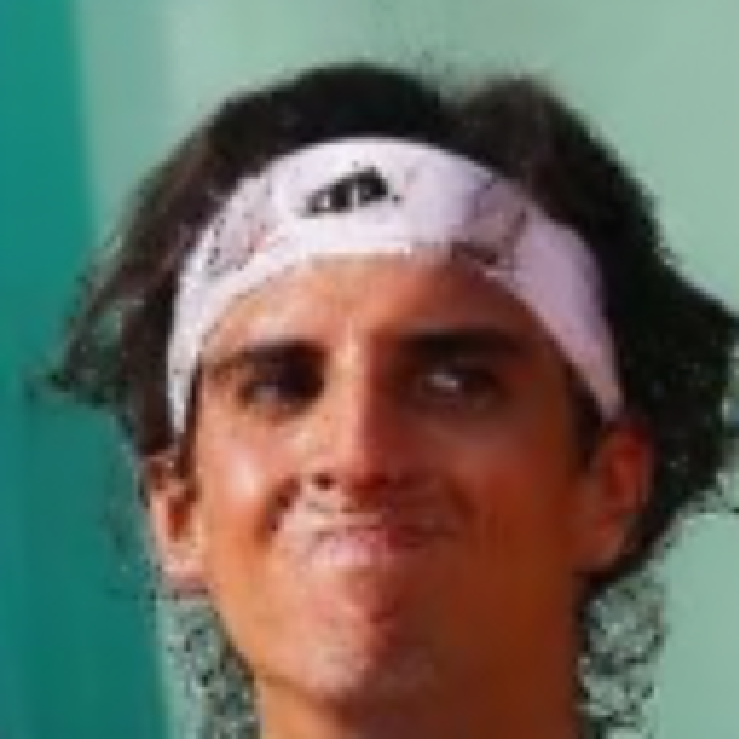} &
  \psnrcell{26.48}{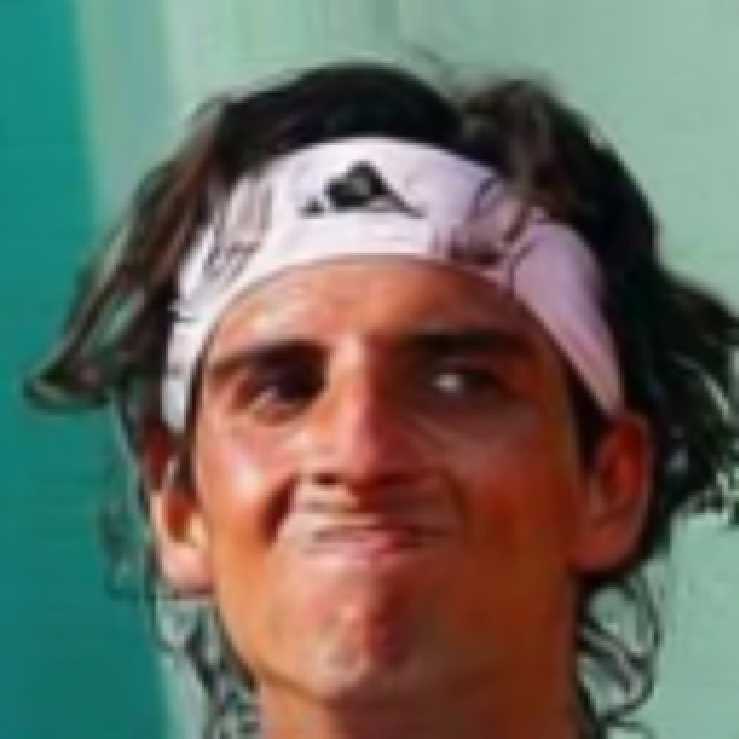} &
  \psnrcell{26.50}{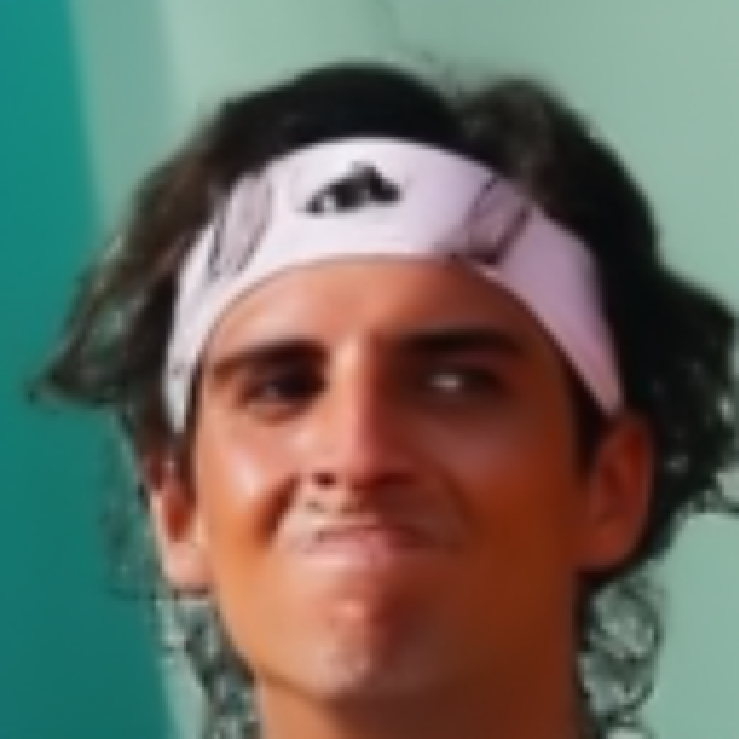} &
  \psnrcell{27.44}{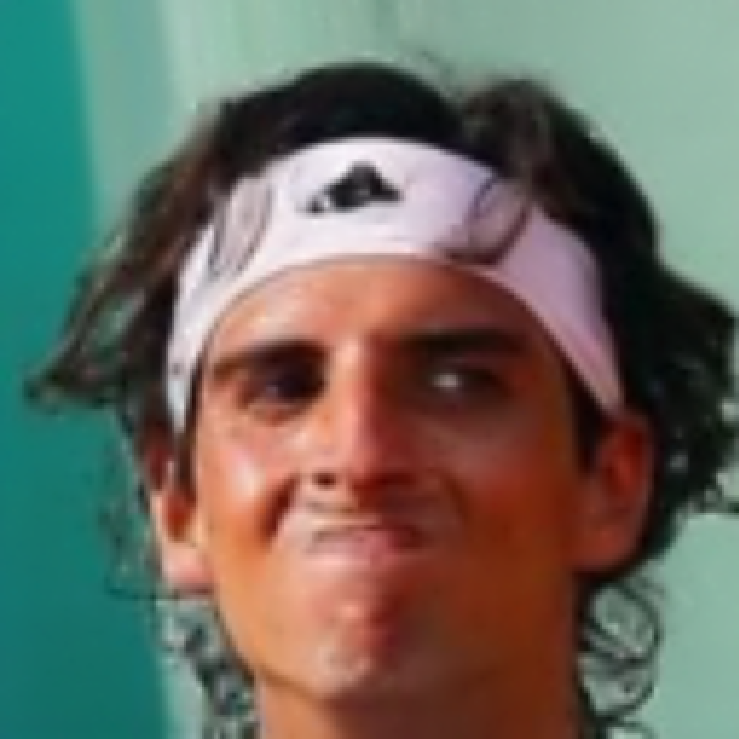} &
  \psnrcell{27.67}{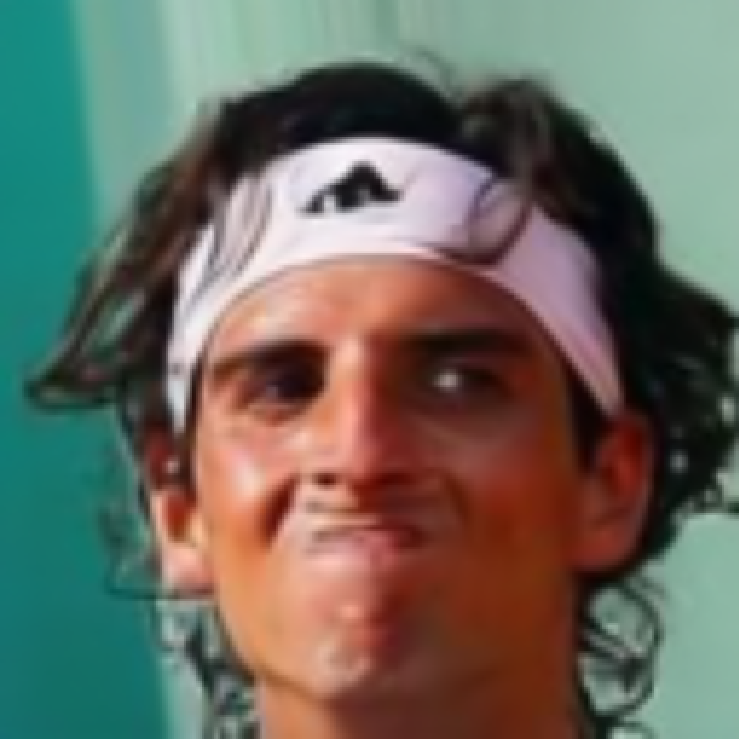} \\
  \bottomrule
  \end{tabular}
  \caption{Example reconstructions for each task. We report the PSNR above each image.}
  \label{fig:result-images}
  \end{figure}

\paragraph{Timing and memory.}
\begin{table}[t]
  \centering
  \caption{Timing and memory metrics per image on the deblurring task on CelebA $(128 \times 128)$ for 100 iterations.}
  \label{tab:time-memory-deblurring}
  \small
 \begin{tabular}{lcc}
  \toprule
  Method & Time / image (s) & Peak memory (GiB) \\
  \midrule
  PnP-Flow5 & 2.84 & 0.30 \\
  Flower5-OT & 5.10 & 0.30 \\
  FlowADMM (sequential averaging) & 2.82 & 0.43 \\
  FlowADMM (batched averaging) & \textbf{1.51} & 1.51 \\
  \bottomrule
  \end{tabular}
\end{table}

We also compare wall-clock time and GPU memory on CelebA Gaussian deblurring,
using the same protocol as in the PnP-Flow paper (100 test images, evaluated as
25 batches of 4 images each). The experiments were conducted on a machine with an i7-8700K CPU and a single RTX 4090 GPU. Table~\ref{tab:time-memory-deblurring} reports the
results for PnP-Flow5, Flower5-OT, and our 100-step FlowADMM configuration. We also report a FlowADMM version that uses batched instead of sequential averaging, which substantially increases performance at the cost of higher peak memory. The speed-up compared to Flower5-OT is due to two factors: FlowADMM performs fewer data term evaluations, and in our implementation these proximal steps are computed in closed form using the structure of the forwarrd operators, whereas Flower uses conjugate gradient iterations (see Appendix~\ref{section:computing-prox-ops}).

\section{Conclusion and Limitations}
\label{sec:conclusions}
We introduced FlowADMM, a PnP ADMM framework based on the mean renoise-denoise operator induced by flow matching models. 
By explicitly formalizing this operator, we derived convergence guarantees for both constant and non-constant timestep schedules under Lipschitz conditions on the underlying flow network. Empirically, FlowADMM achieves state-of-the-art reconstruction quality across a range of inverse problems.
Our work analyzes flow-based PnP methods through an operator-theoretic lens and suggests several directions for future work, including stronger convergence guarantees, and the explicit Lipschitz regularization of flow networks to improve stability.

\paragraph{Limitations.} Our convergence analysis relies on the Lipschitz constant of the underlying flow network that is not explicitly penalized during training. Moreover, the convergence guarantees assume the existence of a fixed point of the limiting FlowADMM operator. In addition, our theoretical analysis focuses on the deterministic mean renoise-denoise operator, whereas the practical Monte Carlo approximation introduces additional stochasticity that is not fully captured by the current theory. Finally, our analysis assumes idealized flow-based denoisers and does not explicitly account for approximation errors from imperfect training.

\paragraph{Broader impact.}
This work focuses on theoretical and algorithmic aspects of solving image inverse problems using flow-based plug-and-play methods, and provides criteria for provable convergence. Potential applications include medical imaging, scientific imaging, and image restoration, where improved reconstruction quality may be beneficial, although limitations in robustness should be carefully considered in safety-critical settings.

% \begin{ack}
% Use unnumbered first level headings for the acknowledgments. All acknowledgments

% Do {\bf not} include this section in the anonymized submission, only in the final paper. You can use the \texttt{ack} environment provided in the style file to automatically hide this section in the anonymized submission.
% \end{ack}

\bibliographystyle{plainnat}
\bibliography{refs}

%%%%%%%%%%%%%%%%%%%%%%%%%%%%%%%%%%%%%%%%%%%%%%%%%%%%%%%%%%%%

\appendix

\section{Proofs}
\label{sec:proof-of-prop1}

\begin{proof}[Proof of Remark~\ref{remark:projection?}]
For $x_1 = x$ we have $x_t = tx + (1-t)x_0$. By the law of total expectation
\begin{equation}
    \bar S_t(x) = \mathbb{E}_{x_0}[D_t(tx_1 + (1-t)x_0) \mid x_1 = x] = \mathbb{E}_{x_0}[\mathbb{E}[x_1 \mid x_t] \mid x_1=x] = x
\end{equation}
\end{proof}

\begin{lemma}
    Let $D_t$ be Lipschitz continuous with constant $L$. Then $\bar{S}_t$ is Lipschitz continuous with constant $tL$.
\end{lemma}
\begin{proof}
    Let $x,y\in\mathbb{R}^d$.
    \begin{equation}
        \bar{S}_t(x) - \bar{S}_t(y) = \mathbb{E}_\epsilon[D_t(tx - (1-t)\epsilon) - D_t(ty - (1-t)\epsilon)]
    \end{equation}
    Taking the norm and applying Jensen's inequality yields
    \begin{align*}
        \Vert \bar{S}_t(x) - \bar{S}_t(y)\rVert &\leq \mathbb{E}_\epsilon [\Vert D_t(tx - (1-t)\epsilon) - D_t(ty - (1-t)\epsilon)\rVert]\\
        &\leq \mathbb{E}_\epsilon[L\Vert tx - (1-t)\epsilon - (ty - (1-t)\epsilon)\rVert]\\
        &= \mathbb{E}_\epsilon[L\Vert tx - ty \rVert]\\
        &= tL \Vert x - y\rVert
    \end{align*}
\end{proof}
\begin{proof}[Proof of Lemma~\ref{lemma:residual_lipschitz}]
    First, assume $(D_t - I)$ is Lipschitz with some constant $L$.
    Let $x,y\in\mathbb{R}^d$ and for brevity define $x_\epsilon = tx + (1-t)\epsilon$, $y_\epsilon = ty + (1-t)\epsilon$ which implies $x_\epsilon - y_\epsilon = t(x-y)$.
    \begin{align*}
        R_t(x) - R_t(y) &= (\bar{S}_t - I)(x) - (\bar{S}_t - I)(y) \\&= \mathbb{E}_\epsilon[D_t(x_\epsilon) - D_t(y_\epsilon)] - (x-y)\\
        &= \mathbb{E}_\epsilon[D_t(x_\epsilon) - D_t(y_\epsilon)] - (t(x-y) + (1-t)(x-y))\\
        &= \mathbb{E}_\epsilon[D_t(x_\epsilon) - D_t(y_\epsilon) - t(x-y)] - (1-t)(x-y)\\
        &= \mathbb{E}_\epsilon[D_t(x_\epsilon) - D_t(y_\epsilon) - (x_\epsilon - y_\epsilon)] - (1-t)(x-y)\\
        &= \mathbb{E}_\epsilon[(D_t - I)(x_\epsilon) - (D_t - I)(y_\epsilon)] - (1-t)(x-y)
    \end{align*}
    Taking the norm and applying Jensen's inequality and the triangle inequality yields
    \begin{align*}
        \Vert (\bar{S}_t &- I)(x) - (\bar{S}_t - I)(y)\rVert \leq \mathbb{E}_\epsilon[\Vert (D_t - I)(x_\epsilon) - (D_t - I)(y_\epsilon) \rVert] + (1-t)\Vert x - y \rVert\\
        &\leq \mathbb{E}_\epsilon[L\Vert t(x-y)\rVert] + (1-t)\Vert x - y \rVert\\
        &= tL \Vert x - y\rVert + (1-t) \Vert x - y\rVert
    \end{align*}
    This means $R_t$ is Lipschitz with constant $tL + (1-t)$.
    If the underlying flow network is $L_v$ Lipschitz then $(D_t - I) = (1-t)v_t^\theta$ is $((1-t)L_v)$-Lipschitz. Setting $L = (1-t)L_v$ in the proof above proves the lemma.
\end{proof}
We can now give a proof of FlowADMM convergence with fixed $t$:
\begin{proof}[Proof of Proposition~\ref{proposition:convergence}]
    From Lemma~\ref{lemma:residual_lipschitz} it follows that $(\bar{S}_t - I)$ is $\xi:=((1-t)(1+tL_v))$-Lipschitz. If $L_v < \frac{1}{1-t}$ then $\xi < 1$. Since $L_v$ is a Lipschitz constant it is also greater or equal to 0 which means $\xi \geq (1-t)$.
    This means that Corollary~3 in \cite{ryu2019plug} is applicable, which proves the convergence of the algorithm under the conditions stated in the proposition. Moreover, the proof in \cite{ryu2019plug} shows that the ADMM iteration is averaged.
\end{proof}

We define $T_{t}$ as the operator defined by a FlowADMM iteration. We just showed that under certain conditions $T_{t}$ is averaged. Assume that $T_t$ has a fixed point. Then the fixed point iteration
\begin{equation}
    x_{k+1} = T_t(x_k)
\end{equation}
converges.
In the following, we will generalize this result to varying $t$-schedules. We start with a more general proposition about general sequences of operators that eventually become averaged.
\begin{proposition}
    \label{proposition:varying-t-convergence}
    Let $(t_k)_k$ be a sequence in $[0, t_\text{max}]$ with $\lim_{k\rightarrow \infty} t_k = t_\text{max} < 1$ and
    \begin{equation}
        \sum_{k=1}^\infty \vert t_k - t_{\text{max}} \rvert < \infty
    \end{equation}
    Let $T_{t_k} : \mathbb{R}^d \to \mathbb{R}^d$ be a family of operators indexed by $t_k$.
    Let $T_* = T_{t_\text{max}}$ be averaged and
    \begin{equation}
        \rVert T_s(w) - T_t(w)\rVert < C\vert s - t\rvert.
    \end{equation}
    Then the iteration
    \begin{equation}
        w_{k+1} = T_{t_k}(w_{k})
    \end{equation}
    converges to a fixed point of $T_*$, if such a fixed point exists.
\end{proposition}
\begin{proof}
    Denote
    \begin{equation}
        e_k = T_{t_k}(w_k) - T_*(w_k).
    \end{equation}
    We have
    \begin{equation}
        \Vert e_k \rVert = \Vert T_{t_k}(w_k) - T_*(w_k) \rVert \leq C\vert t_k - t_\text{max}\rvert
    \end{equation} and thus 
    \begin{equation}
        \sum_{k=1}^\infty \Vert e_k\rVert < \infty
    \end{equation}
    from the assumptions. Since $T_*$ is averaged, we can write
    \begin{equation}
        T_* = (1-\lambda)I + \lambda N
    \end{equation}
    with some non-expansive operator $N$.
    Thus,
    \begin{align}
        T_{t_k}(w_k) &= T_*(w_k) + e_k\\ 
                     &= w_k + \lambda (N(w_k) - w_k) + e_k\\
                     &= w_k + \lambda (N(w_k) + \frac{e_k}{\lambda} - w_k)
    \end{align}
    This form allows us to directly apply Proposition 5.34 of \citet{Bauschke2017}. With $(\lambda_k)_k = \lambda$ we get
    \begin{equation}
        \sum_{k=1}^\infty \lambda_k \Vert\frac{e_k}{\lambda}\rVert = \sum_{k=1}^\infty \Vert e_k\rVert < \infty.
    \end{equation}
    The convergence to a fixed point now directly follows from Proposition 5.34 in \cite{Bauschke2017} and the fact that $T_*$ has the same fixed points as $N$.
\end{proof}

These results can finally be used to show the convergence of FlowADMM with varying $t$.
\begin{proof}[Proof of Proposition~\ref{proposition:varying-t-admm-convergence}]
    From the assumptions of Proposition~\ref{proposition:convergence} $v_t^\theta(x)$ is already Lipschitz in $x$. To apply Proposition~\ref{proposition:varying-t-convergence} we need to show that $T_t(w)$ is Lipschitz in $t$. Together with the new condition on the Lipschitzness in $t$ on the flow network we have
    \begin{equation}
        \Vert v_s^\theta (x) - v_t^\theta(y)\rVert \leq L_x \Vert x - y\rVert + L_t\vert s - t\rvert
    \end{equation}
    Clearly the denoiser $D_t(x) = x + (1-t)v_t^\theta(x)$ is also Lipschitz in $x$ and $t$, with constants $C_x$ and $C_s$. For the mean denoising operator we have
    with $x_\alpha = \alpha x + (1-\alpha)\epsilon$ we have
    \begin{align}
        \Vert \bar S_t(x) - \bar S_s(x)\rVert &= \Vert \mathbb{E}[D_t(x_t)] - \mathbb{E}[D_s(x_s)]\rVert \\
        &= \Vert \mathbb{E}[D_t(x_t) - D_t(x_s)] + \mathbb{E}[D_t(x_s) - D_s(x_s)] \rVert \\
        &\leq \mathbb{E} [\Vert D_t(x_t) - D_t(x_s)\rVert] + \mathbb{E}[\Vert D_t(x_s) - D_s(x_s)\rVert]\\
        &\leq C_x\mathbb{E} [\Vert x_t - x_s \rVert] + C_t|s - t \rvert.
    \end{align}
    Since $x_t - x_s = (t-s)(x - \epsilon)$ we get
    \begin{equation}
        \mathbb{E}[\Vert x_t - x_s\rVert] = \vert t-s\rvert(\mathbb{E}[\Vert(x - \epsilon)\rVert].
    \end{equation}
    From the assumption that $T_{t_k}$ becomes averaged in a $\delta$-neighborhood of $t_\text{max}$, it follows that it becomes averaged after a finite number of steps and the sequence of iterates is therefore bounded.
    Therefore, $\mathbb{E}[\Vert(x - \epsilon)\rVert]$ is also bounded, so 
    \begin{equation}
        \mathbb{E}[\Vert x_t - x_s\rVert] < C_\epsilon \vert t - s \rvert
    \end{equation}
    and in total we have
    \begin{equation}
        \Vert \bar S_t(x) - \bar S_s(x)\rVert \leq (C_xC_\epsilon + C_t)\vert s - t\rvert.
    \end{equation}
    Since the only ADMM step that depends on $t$ is the application of $\bar S_t$, the $T_*$ operator inherits the Lipschitzness from $\bar S_t$. The rest follows immediately from Proposition~\ref{proposition:varying-t-convergence}
\end{proof}

\section{Experiments on the Lipschitz condition}
\label{sec:experiments-lipschitz}
To better understand the stability of the mean denoiser in FlowADMM, we estimate the Lipschitz constants of the underlying flow network by analyzing the distribution of Jacobian spectral norms along the late-stage trajectory of FlowADMM for a denoising task
\begin{equation}
    \hat L_v(x, t) \approx \Vert J_x v_t^\theta(x)\Vert_2
\end{equation}
We estimated this quantity at denoiser inputs $\tilde z_k = (t(x_{k+1} + u_k) + (1-t)\epsilon)$ on the last 10 of the 100 stages (corresponding to $t>0.9$) of FlowADMM evaluated on CelebA. We estimate the Jacobian norms by matrix free power iterations on $(J_x v_t(\tilde z_k))^\top J_x v_t(\tilde z_k)$, implemented with forward and reverse mode autodiff in PyTorch, which efficiently compute jacobian-vector-products (jvp) and vector-jacobian-products (vjp). Specifically, we iterate
\begin{align}
    \tilde w_{k+1} &= \frac{J w_k}{\Vert J w_k \rVert } \quad \text{(using jvp),}\\
    w_{k+1} &= \frac{J^\top \tilde{w}_{k+1}}{\Vert J^\top \tilde{w}_{k+1}\Vert}\quad \text{(using vjp)}.
\end{align}
We plot the distributions of the $(1-t)\hat L_v(\tilde z_k,t)$ estimates in histograms for each evaluated timestep in Figure~\ref{fig:lipschitz-histograms}. This quantity corresponds to the Lipschitz condition in Proposition~\ref{proposition:convergence}, where values smaller than 1 suggest stability.
\begin{figure}
    \centering
    \begin{tabular}{ccc}
         \includegraphics[width=0.3\linewidth]{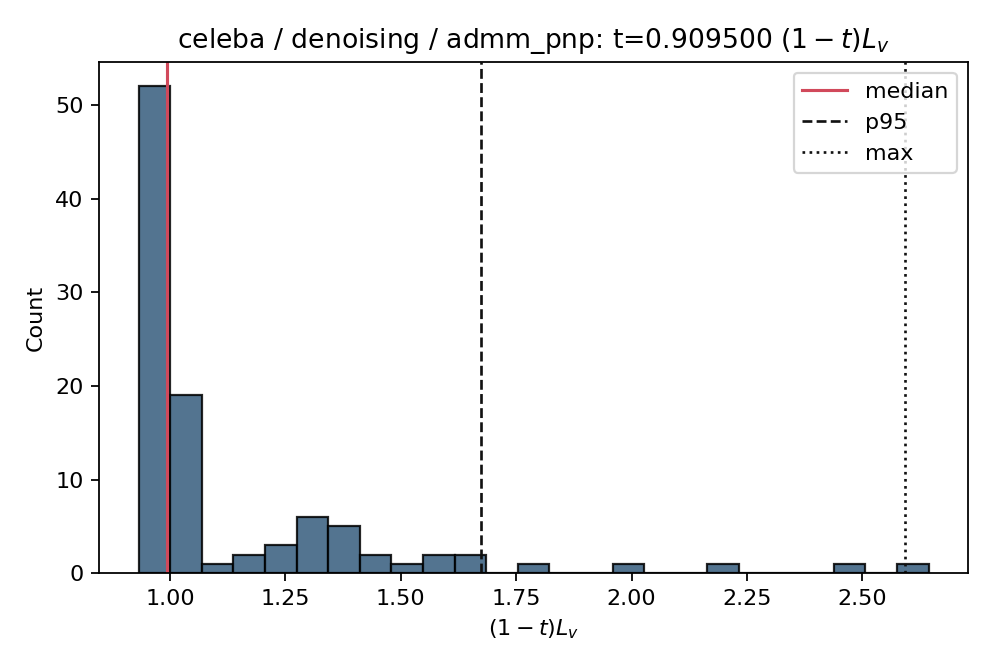} &
         \includegraphics[width=0.3\linewidth]{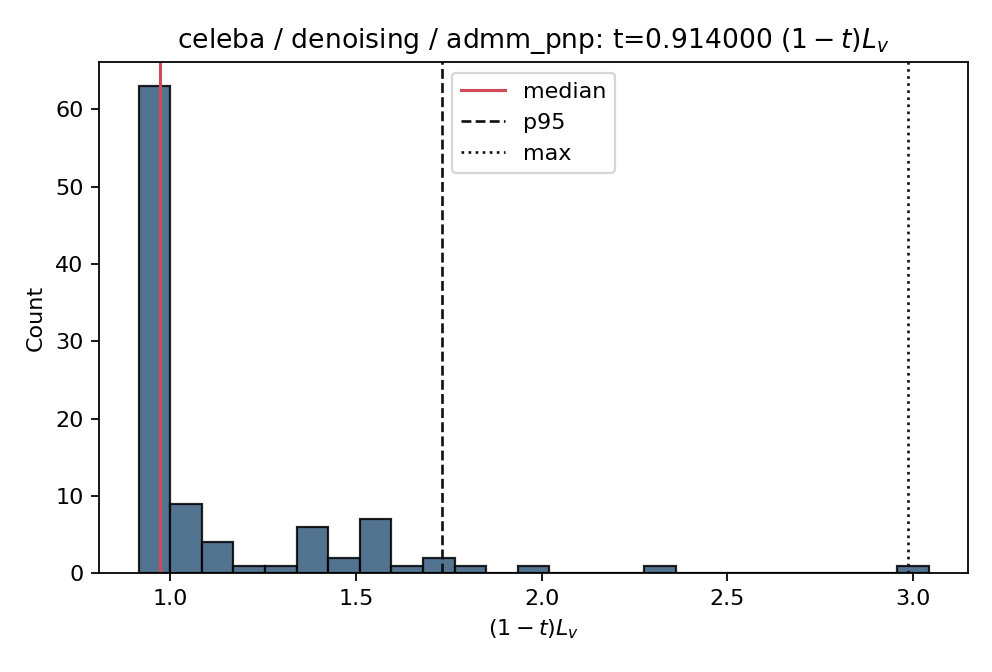} &
         \includegraphics[width=0.3\linewidth]{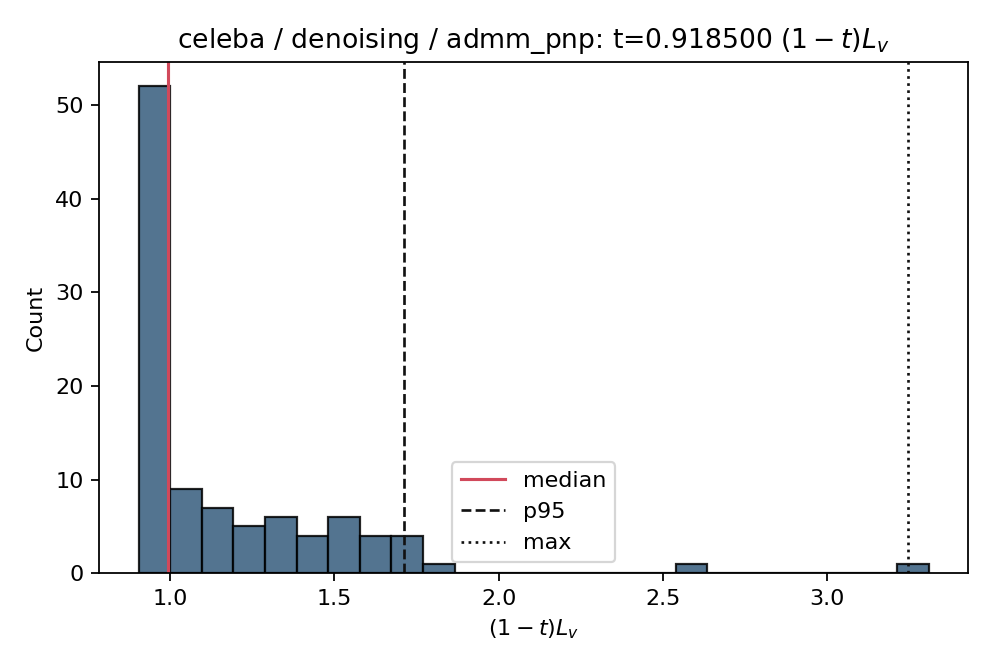} \\
         \includegraphics[width=0.3\linewidth]{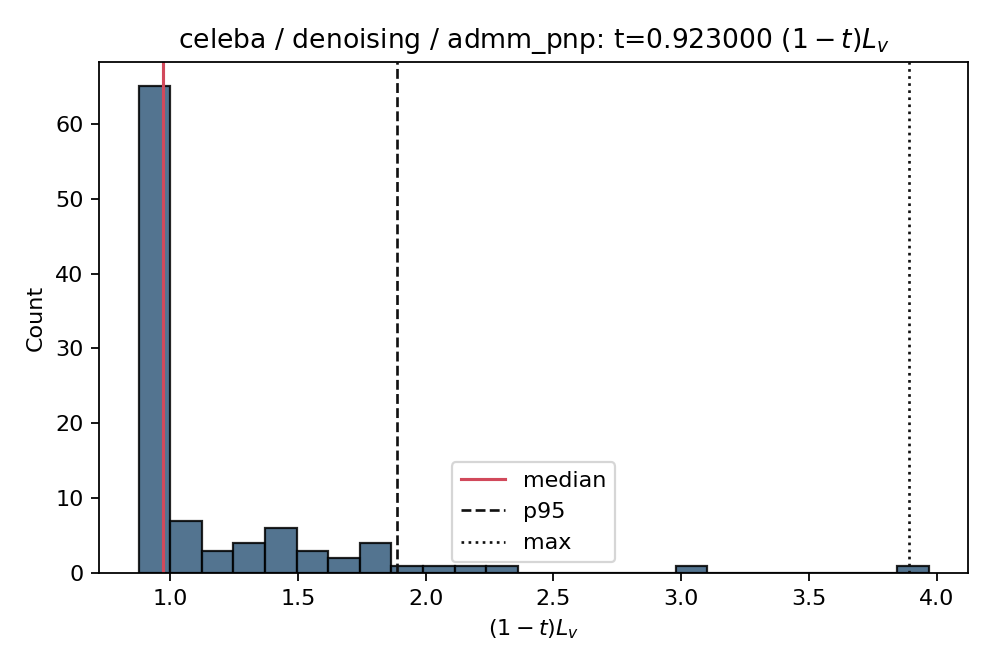} &
         \includegraphics[width=0.3\linewidth]{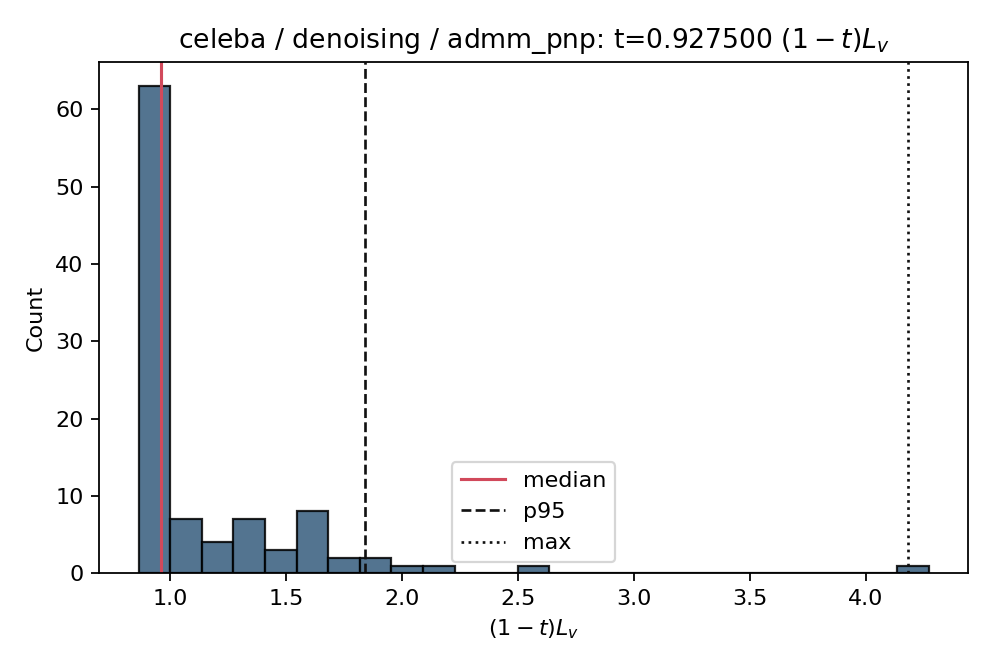} &
         \includegraphics[width=0.3\linewidth]{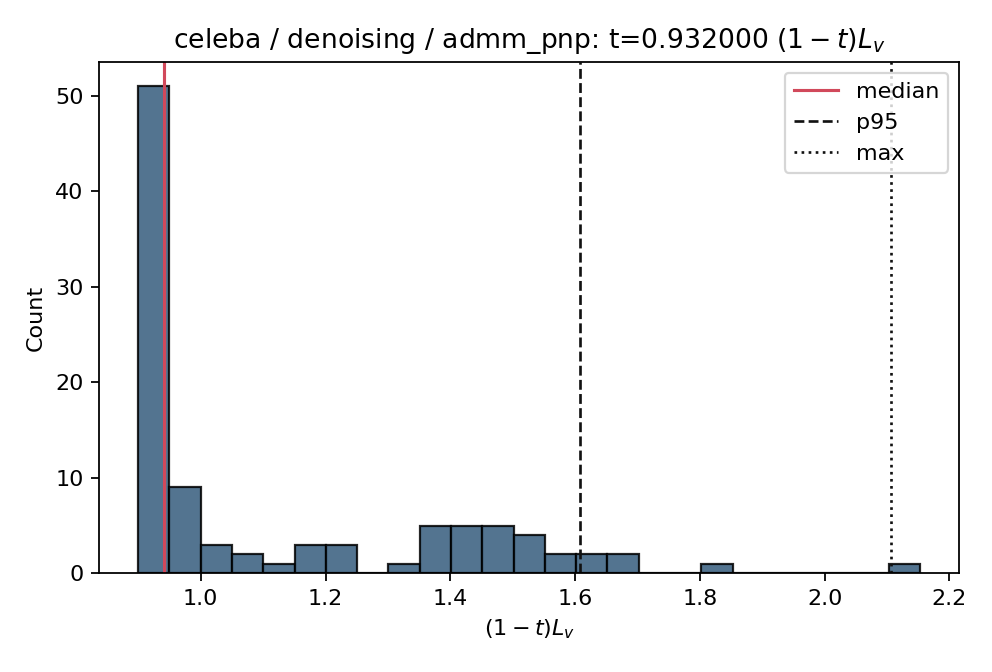} \\
         \includegraphics[width=0.3\linewidth]{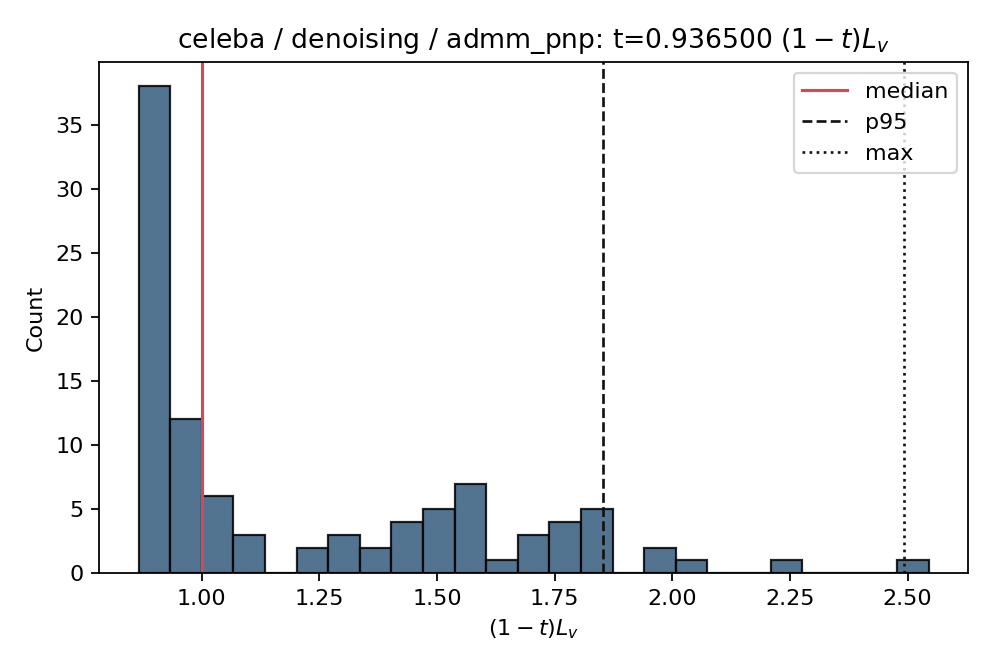} &
         \includegraphics[width=0.3\linewidth]{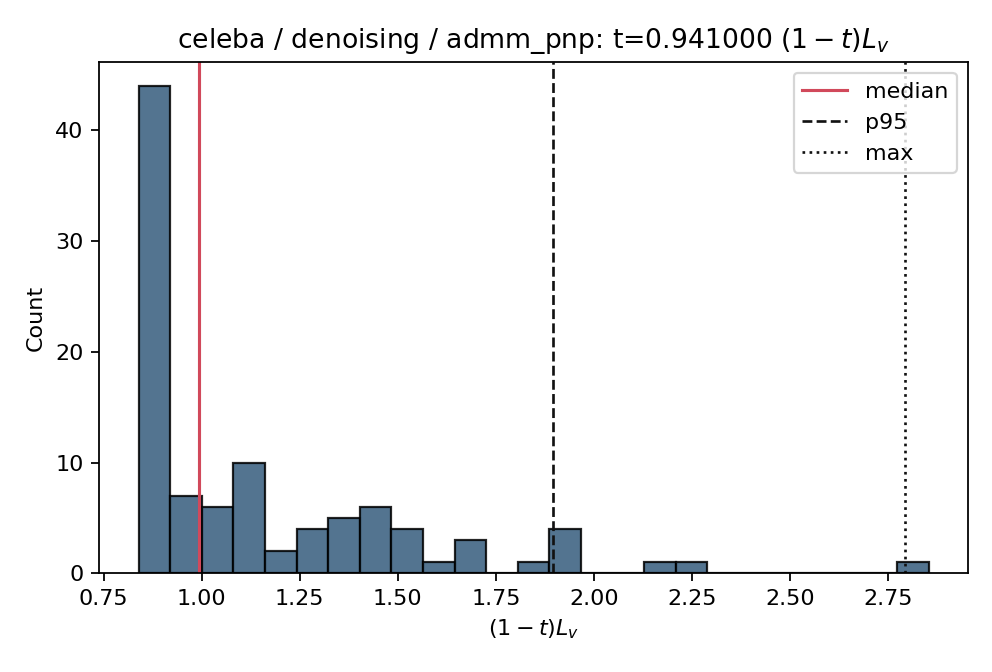} &
         \includegraphics[width=0.3\linewidth]{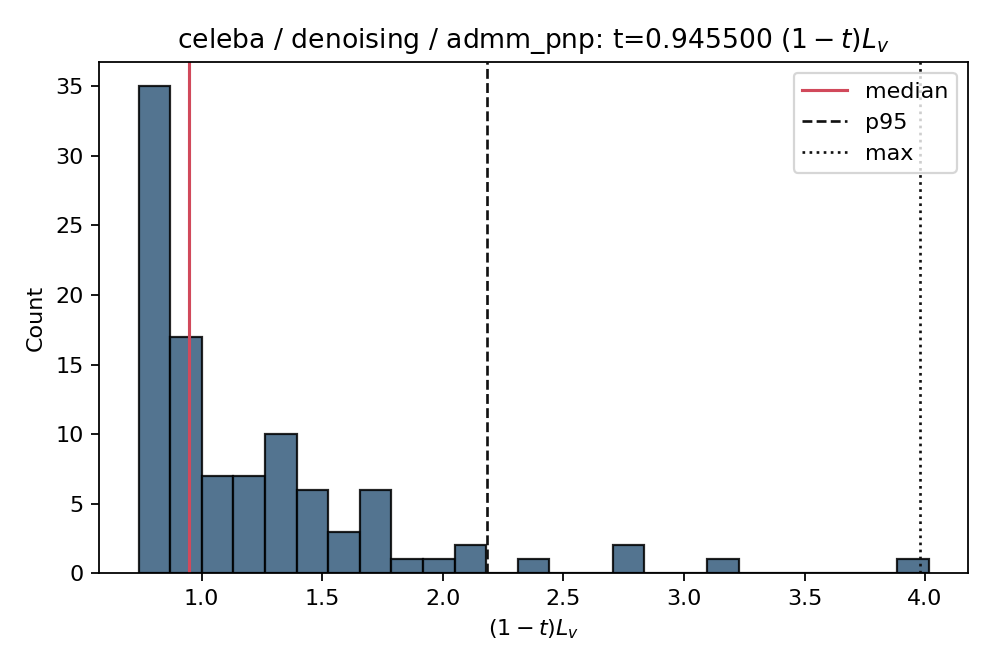} \\
         \includegraphics[width=0.3\linewidth]{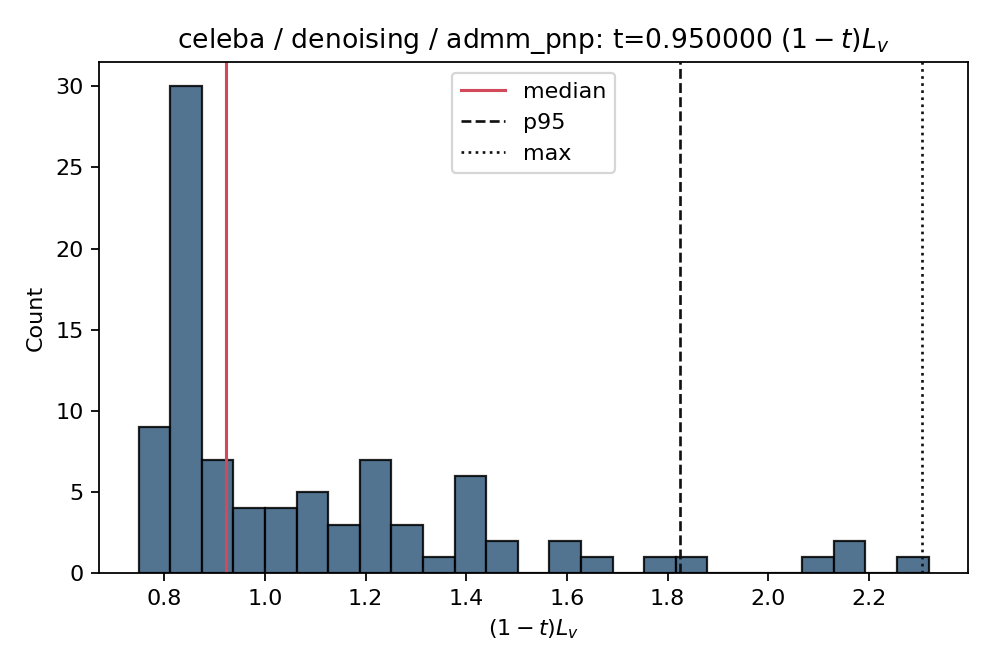}
    \end{tabular}
    \caption{Distribution of the Jacobian norm Lipschitz estimates along the late-stage FlowADMM trajectory on the denoising task.}
    \label{fig:lipschitz-histograms}
\end{figure}
The histograms in Figure~\ref{fig:lipschitz-histograms} shows that the median of $(1-t)\hat L_v$ is below $1$.
This suggests that FlowADMM trajectories largely stay in stable regions.
Nevertheless, the distributions also show long tails with clear spikes above $1$, which shows that convergence using Proposition~\ref{proposition:convergence} cannot be guaranteed without enforcing a hard Lipschitz constraint on the underlying flow network.
Note, that the experiments in this section are only heuristics and we do not claim to compute the global Lipschitz constant of the flow network. Finding a tight upper bound for the global Lipschitz constant of a neural network is in general hard to do~\cite{virmaux2018lipschitz}.

\section{Computing proximal operators}
\label{section:computing-prox-ops}
In this work, we consider quadratic data fidelity terms of the form
\begin{equation}
F_y(x) = \frac{1}{2}\Vert Ax - y\rVert_2^2,
\end{equation}
whose proximal operator is given in closed form by
\begin{equation}
\prox_{\mu F_y}(v)
=
(I + \mu A^\top A)^{-1}(v + \mu A^\top y).
\end{equation}
For large-scale problems, explicitly forming or inverting the matrix $(I + \mu A^\top A)$ is typically infeasible, and matrix-free iterative solvers are often employed.

However, many linear operators in imaging applications have a decomposition
\begin{equation}
A = Q^\top \Lambda P,
\end{equation}
where $P$ and $Q$ are orthogonal transforms and $\Lambda$ is diagonal. This includes, for example, convolution operators (diagonalized by the discrete Fourier transform), masking operators (diagonal in the pixel domain), and subsampling operators.

Using this decomposition, we obtain
\begin{equation}
A^\top A = P^\top \Lambda^\top \Lambda P,
\qquad
A^\top y = P^\top \Lambda^\top Q y,
\end{equation}
which allows rewriting the proximal operator in the transformed domain (left multiplying with $P$) as
\begin{equation}
P\prox_{\mu F_y}(v)
=
(I + \mu \Lambda^\top \Lambda)^{-1}
\left(P v + \mu \Lambda^\top Q y\right).
\end{equation}
Defining $z = P\prox_{\mu F_y}(v)$, the problem reduces into element-wise equations that can be solved by
\begin{equation}
z_i
=
\frac{(P v)_i + \mu \lambda_i (Q y)_i}
{1 + \mu \lambda_i^2},
\end{equation}
where $\lambda_i$ denotes the $i$-th diagonal entry of $\Lambda$. The solution is then obtained via the inverse transform
\begin{equation}
\prox_{\mu F_y}(v) = P^\top z.
\end{equation}
We use this formulation wherever possible and fall back to a CG solver where such a decomposition is not available.

 \section{Task Details and Forward Operators}
  \label{sec:task-details}
  In the main paper, we evaluated on the same task settings as \cite{martin2025pnp} and \cite{pourya2025flower}. In all cases, the measurement model is
\begin{equation}
  y = Ax + \epsilon,
\end{equation}
  where $A$ is the task-specific forward operator and $\epsilon$ is additive Gaussian noise with standard deviation $\sigma_\text{noise}$. The implementation of the forward operators and their transposes was taken without change from the official PnP-Flow code~\cite{martin2025pnp}, which will be detailed in the following paragraphs.

  \paragraph{Denoising.}
  For denoising, the forward operator is the identity,
  $A = I$.
  Thus the measurement is simply a noisy version of the clean image. We use Gaussian noise with standard deviation $\sigma_{\mathrm{noise}}=0.2$.

  \paragraph{Gaussian deblurring.}
  For deblurring, $A$ is convolution with a Gaussian blur kernel implemented in the Fourier domain. The transposed operator $A^\top$ is implemented as convolution with the conjugate filter in the FFT representation. We use a kernel size of $61\times 61$, with blur standard
  deviation $\sigma_{\mathrm{blur}}=1.0$ on CelebA and $\sigma_{\mathrm{blur}}=3.0$ on AFHQ-Cat, followed by additive Gaussian noise with $\sigma_{\mathrm{noise}}=0.05$.

  \paragraph{Super-resolution.}
  For super-resolution, $A$ is a strided subsampling operator that keeps one pixel every $s$ pixels in each spatial direction. Its adjoint $A^\top$ is the corresponding zero-filled upsampling operator. We use scale factor $s=2$ on CelebA and $s=4$ on
  AFHQ-Cat, with additive Gaussian noise $\sigma_{\mathrm{noise}}=0.05$.

  \paragraph{Box inpainting.}
  For box inpainting, $A$ is a binary masking operator that removes a centered square region, and $A^\top=A$ since the mask is diagonal. The mask half-size is $20$ pixels on CelebA and $40$ pixels on AFHQ-Cat, corresponding to a missing square of side
  length $40$ and $80$, respectively. We add Gaussian noise with $\sigma_{\mathrm{noise}}=0.05$.

  \paragraph{Random inpainting.}
  For random inpainting, $A$ is a Bernoulli masking operator with independent probability of missing pixels of $p=0.7$, so that on average $70\%$ of pixels are removed. Again, the operator is symmetric, i.e.\ $A^\top=A$. The random mask is fixed across runs
  using a fixed random seed. We add Gaussian noise with $\sigma_{\mathrm{noise}}=0.01$ on the observed pixels.

\section{Hyperparameter tuning details}
\label{sec:hyperparameter-details}
\begin{algorithm}[ht]
\caption{FlowADMM with Monte Carlo Sampling}
\label{alg:flowadmm}
\begin{algorithmic}[1]
\Require Measurements $y$, forward operator $A$, pretrained flow denoiser $D_t$, iterations $K$, data term step size $\tau$, time schedule $(t_k)_k$, sample schedule $(N_k)_k$
\State Initialize $z_0 = A^\top y$, $x_0 = A^\top y$, $u_0 = 0$
\For{$k=0,\ldots,K-1$}
    \State $x_{k+1} \gets \prox_{\tau F_y}(z_k-u_k)$
    
    \State $z_{k+1} \gets 0$
    \For{$i=1,\ldots,N_k$}
        \State Sample $\epsilon_i \sim \mathcal{N}(0,I)$
        \State $z_{k+1} \gets z_{k+1} + \frac{1}{N_k} D_{t_k}\!\left(t_k(x_{k+1}+u_k) + (1-t_k)\epsilon_i\right)$
    \EndFor

    \State $u_{k+1} \gets u_k + x_{k+1} - z_{k+1}$
\EndFor
\State \Return $z_K$
\end{algorithmic}
\end{algorithm}

We explored a 3-phase schedule and a constant schedule for $N_k$.
The three phase schedule is defined as
\begin{equation}
    N_k = \begin{cases}
        N_e, & \frac{k}{K} < s_1,\\
        N_m, & s_1 \leq \frac{k}{K} < s_2,\\
        N_l & \frac{k}{K} \geq s_2,
    \end{cases}
\end{equation}
We denote a three phase schedule with the above parameters as $\text{3ph}(N_e, N_m, N_l; s_1, s_2)$ and "const" means $N_k = 5$ for all $k$.
The sweep over the three-phase schedule parameters was done in such a way that the total number of $N_k$ stays constant. This ensures that FlowADMM uses the same number of flow evaluations as Flower5-OT and PnP-Flow5.
We want to mention that we mainly explored 3-phase over 2-phase or other schedules to be more easily able to satisfy that constraint, and mainly explored schedules with low $N_m$, i.e. the mid-phase was used to "fill up" the remaining sample budget. In real world settings more complicated schedules are possible and might even lead to better performance.
We did a grid search with the following values
\begin{itemize}
    \item $\tau \in \{0.125,\,0.25,\,0.5,\,1.0,\,2.0,\,5.0,\,10.0\}$
    \item $t_\text{min} \in \{0.1,\,0.2,\,0.3,\,0.5\}$
    \item $t_\text{max} \in \{0.9,\,0.95\}$
    \item $\gamma \in \{0.5,\,1.0,\,2.0\}$
    \item $s_1 \in \{0.5,\,0.6,\,0.7\}$
    \item $s_2 \in \{0.8,\,0.9\}$
    \item $N_e = 1$ (fixed early phase)
    \item $N_m \in \{1,\,3,\,4\},$
    \item $N_l$ was chosen to fix the flow evaluation budget constraint of matching Flower5-OT's and PnP-Flow5's flow evalutions, i.e. $K\cdot N_\text{avg}$ flow evaluations in total
\end{itemize}

Note, that while all CelebA experiments use $K=100$ iterations, some experiments on AFHQ-Cat use more, namely super-resolution with 500 and random inpainting with 200 iterations.
 \begin{table*}[t]
  \centering
  \caption{Hyperparameters used for the CelebA experiments. For baselines, we report the tuned parameters computed in \cite{pourya2025flower}}
  \label{tab:hparams-celeba}
  \scriptsize
  \resizebox{\textwidth}{!}{%
  \begin{tabular}{llccccc}
  \toprule
  Method & Hyperparameter & Denoising & Deblurring & Super-resolution & Random inpainting & Box inpainting \\
  \midrule

  \multirow{2}{*}{OT-ODE}
  & $t_0$   & 0.3 & 0.4 & 0.1 & 0.1 & 0.1 \\
  & $\gamma$ & td  & td  & const & const & td \\
  \midrule

  \multirow{3}{*}{D-Flow}
  & $\lambda$ & 0.001 & 0.001 & 0.001 & 0.001 & 0.001 \\
  & $\alpha$  & 0.1   & 0.1   & 0.1   & 0.1   & 0.1 \\
  & $n_{\text{iter}}$ & 3 & 7 & 10 & 20 & 9 \\
  \midrule

  \multirow{2}{*}{Flow-Priors}
  & $\lambda$ & 100 & 1000 & 10000 & 10000 & 10000 \\
  & $\eta$    & 0.01 & 0.01 & 0.1 & 0.01 & 0.01 \\
  \midrule

  \multirow{3}{*}{PnP-Flow5}
  & $K$ & 100 & 100 & 100 & 100 & 100 \\
  & $N_{\text{avg}}$ & 5 & 5 & 5 & 5 & 5 \\
  & $\alpha$ & 0.8 & 0.01 & 0.3 & 0.01 & 0.5 \\
  \midrule

  \multirow{2}{*}{Flower5-OT}
  & $K$ & 100 & 100 & 100 & 100 & 100 \\
  & $N_{\text{avg}}$ & 5 & 5 & 5 & 5 & 5 \\
  \midrule

  \multirow{7}{*}{FlowADMM}
  & $K$ & 100 & 100 & 100 & 100 & 100 \\
  & $\tau$ & 5.0 & 0.5 & 0.5 & 0.25 & 1.0 \\
  & $t_{\min}$ & 0.5 & 0.5 & 0.3 & 0.3 & 0.1 \\
  & $t_{\max}$ & 0.95 & 0.95 & 0.95 & 0.95 & 0.95 \\
  & $\gamma$ & 1.0 & 0.5 & 1.0 & 0.5 & 2.0 \\
  & $N_k$ &
  $\text{3ph}(1,1,41;0.5,0.9)$ &
  $\text{3ph}(1,1,41;0.5,0.9)$ &
  $\text{3ph}(1,3,35;0.6,0.9)$ &
  $\text{3ph}(1,4,29;0.5,0.9)$ &
  $\text{3ph}(1,4,35;0.7,0.9)$ \\
  \bottomrule
  \end{tabular}}
  \end{table*}

  \begin{table*}[t]
  \centering
  \caption{Hyperparameters used for the AFHQ-Cat experiments. For baselines, we report the tuned parameters computed in \cite{pourya2025flower}}
  \label{tab:hparams-afhq}
  \scriptsize
  \resizebox{\textwidth}{!}{%
  \begin{tabular}{llccccc}
  \toprule
  Method & Hyperparameter & Denoising & Deblurring & Super-resolution & Random inpainting & Box inpainting \\
  \midrule

  \multirow{2}{*}{OT-ODE}
  & $t_0$   & 0.3 & 0.3 & 0.1 & 0.1 & 0.1 \\
  & $\gamma$ & td  & td  & const & const & td \\
  \midrule

  \multirow{2}{*}{Flow-Priors}
  & $\lambda$ & 100 & 1000 & 10000 & 10000 & 10000 \\
  & $\eta$    & 0.01 & 0.01 & 0.1 & 0.01 & 0.01 \\
  \midrule

  \multirow{3}{*}{PnP-Flow5}
  & $K$ & 100 & 500 & 500 & 200 & 100 \\
  & $N_{\text{avg}}$ & 5 & 5 & 5 & 5 & 5 \\
  & $\alpha$ & 0.8 & 0.01 & 0.01 & 0.01 & 0.5 \\
  \midrule

  \multirow{2}{*}{Flower5-OT}
  & $K$ & 100 & 100 & 500 & 200 & 100 \\
  & $N_\text{avg}$ & 5 & 5 & 5 & 5 & 5 \\
  \midrule

  \multirow{7}{*}{FlowADMM}
    & $K$ & 100 & 100 & 500 & 200 & 100 \\
    & $\tau$ & 5.0 & 0.25 & 0.25 & 0.125 & 0.5 \\
    & $t_{\min}$ & 0.5 & 0.5 & 0.3 & 0.3 & 0.1 \\
    & $t_{\max}$ & 0.95 & 0.95 & 0.95 & 0.95 & 0.9 \\
    & $\gamma$ & 1.0 & 0.5 & 1.0 & 0.5 & 2.0 \\
    & $N_k$ &
    $\mathrm{3ph}(1,1,41;0.5,0.9)$ &
    const &
    $\mathrm{3ph}(1,4,29;0.5,0.9)$ &
    $\mathrm{3ph}(1,3,33;0.5,0.9)$ &
    $\mathrm{3ph}(1,3,19;0.6,0.8)$ \\
    \bottomrule
    \end{tabular}}

  \end{table*}

\section{Ablations}
\label{sec:ablations}
\subsection{Influence of hyperparameters}
\begin{table}[t]
  \centering
  \caption{Ablation of FlowADMM on the CelebA validation set for the five inverse problems. We report mean PSNR over the five tasks. The top row is the final method used in the
  paper, and all subsequent rows are measured relative to it.}
  \label{tab:flowadmm_ablation}
  \scriptsize
  \begin{tabular}{lccc}
  \toprule
  Setting & Mean PSNR $\uparrow$ & $\Delta$ vs final method & Max single-task drop \\
  \midrule
  Final method & \textbf{34.1815} & - & - \\
  Shared $\mathrm{3ph}(1,3,18;0.5,0.8)$ & 34.1527 & -0.0288 & -0.0484 \\
  Constant $N_k=5$ & 33.9064 & -0.2751 & -0.3499 \\
  Shared $t_{\min}=0.3$ & 34.1396 & -0.0419 & -0.1508 \\
  Shared $\gamma=1.0$ & 34.1198 & -0.0617 & -0.2594 \\
  \bottomrule
  \end{tabular}
  \end{table}
We evaluate the sensitivity of our hyperparameters and whether they can be shared between tasks. Table~\ref{tab:flowadmm_ablation} shows that the main gains come from late-heavy three-phase averaging rather than from fine-tuning the time schedule. Relative to the final per-task tuned
method, replacing the task-specific schedule by the shared three-phase schedule $\mathrm{3ph}(1,3,18;0.5,0.8)$ costs only 0.0288 dB on average, whereas reverting to constant averaging
loses 0.2751 dB, showing that the principled late-heavy schedule leads to significant performance gains. We also find that sharing $t_{\min}=0.3$ or $\gamma=1.0$ incurs modest average drops of 0.0419 dB and 0.0617
dB, respectively, with the largest degradation concentrated on box inpainting.
These results suggest that most hyperparameters except $\tau$ can be unified between tasks, incuring only modest performance penalties.
\subsection{Applying the late-heavy sampling schedule to PnP-Flow}
\label{sec:late-heavy-pnp-flow}
\begin{table}[t]
    \centering
    \caption{CelebA test-set per-task PSNR comparison between the standard PnP-Flow5 baseline, using the late-heavy $N_k$ schedule in PnP-Flow (PnP-Flow-KAvg), and FlowADMM. Positive deltas indicate improvements over the corresponding reference method.}
    \label{tab:pnpflow-kavg-vs-flowadmm-taskwise}
    \small
    \begin{tabular}{lccccc}
    \toprule
    Task & PnP-Flow5 & PnP-Flow-KAvg & FlowADMM & $\Delta$ vs PnP-Flow5 & $\Delta$ vs FlowADMM \\
    \midrule
    Denoising & 32.3119 & 32.3972 & 33.2049 & +0.0853 & -0.8077 \\
    Deblurring & 34.8035 & 34.8616 & 35.6008 & +0.0581 & -0.7392 \\
    Super-resolution & 31.5030 & 33.3271 & 34.0372 & +1.8241 & -0.7101 \\
    Box inpainting & 31.0589 & 31.5672 & 31.7281 & +0.5083 & -0.1609 \\
    Random inpainting & 34.0700 & 34.2278 & 35.4008 & +0.1578 & -1.1730 \\
    \bottomrule
    \end{tabular}
  \end{table}

We further investigate the influence of the late-heavy $N_k$ schedule by also applying it to PnP-Flow. We tune the schedule together with other PnP-Flow hyperparameters for each task on CelebA.
Table~\ref{tab:pnpflow-kavg-vs-flowadmm-taskwise} shows that late-heavy sampling schedule improves PnP-Flow on all five CelebA tasks, with the largest gains on super-resolution and box inpainting. However, FlowADMM still remains better on every
task, with the remaining gap ranging from 0.1609 dB on box inpainting to 1.1730 dB on random inpainting. This suggests that the $N_k$ schedule is inherently advantageous for renoise-denoise methods, while also giving evidence that ADMM has an advantage over forward-backward splitting methods for flow-based PnP methods.

\section{Additional results}
\subsection{Trajectory visualizations}
Figures~\ref{fig:traj_deblurring}, \ref{fig:traj_superres}, \ref{fig:traj_rand_inp} and \ref{fig:traj_box_inp} visualize the FlowADMM trajectory through the primal data-consistency iterate $x$, the prior iterate $z$, and the dual variable $u$. In the first iterations, $x$ and $z$ can differ a lot. $x$ reacts directly to the measurement model through the proximal update, whereas $z$ stays closer to the flow prior. As the iterations proceed, the gap between the two iterates shrinks and both converge to the same reconstruction, indicating that the consensus constraint is being satisfied. This highlights an advantage of the ADMM over proximal gradient descent: data consistency and prior regularization are handled by separate, interpretable states rather than being entangled in a single update.

This separation is especially useful in the early phase of the reconstruction, where the observation model and the learned prior may disagree strongly. FlowADMM allows each component to make large task-specific corrections before forcing consensus, which is not directly visible in single-state proximal gradient descent trajectories. Note that we initialize $x = z = A^\top y$.

Toward the end of the iteration the dual variable $u$ becomes 0. This can be explained with the following.
Near the end of the algorithm $\bar S_t(x) \approx x$. This happens because $t$ gets closer to 1 (i.e. the prior gets weaker) and from Remark~\ref{remark:projection?} because $x$ is near the manifold. This means the $z$ and $u$ updates become
\begin{align}
    z_{k+1} &= x_{k+1} + u_k\\
    u_{k+1} &= u_k + x_{k+1} - (x_{k+1} + u_k) = 0,
\end{align}
so $u$ naturally vanishes in late iterations.

\begin{figure}[h]
    \centering
    \includegraphics[width=\linewidth]{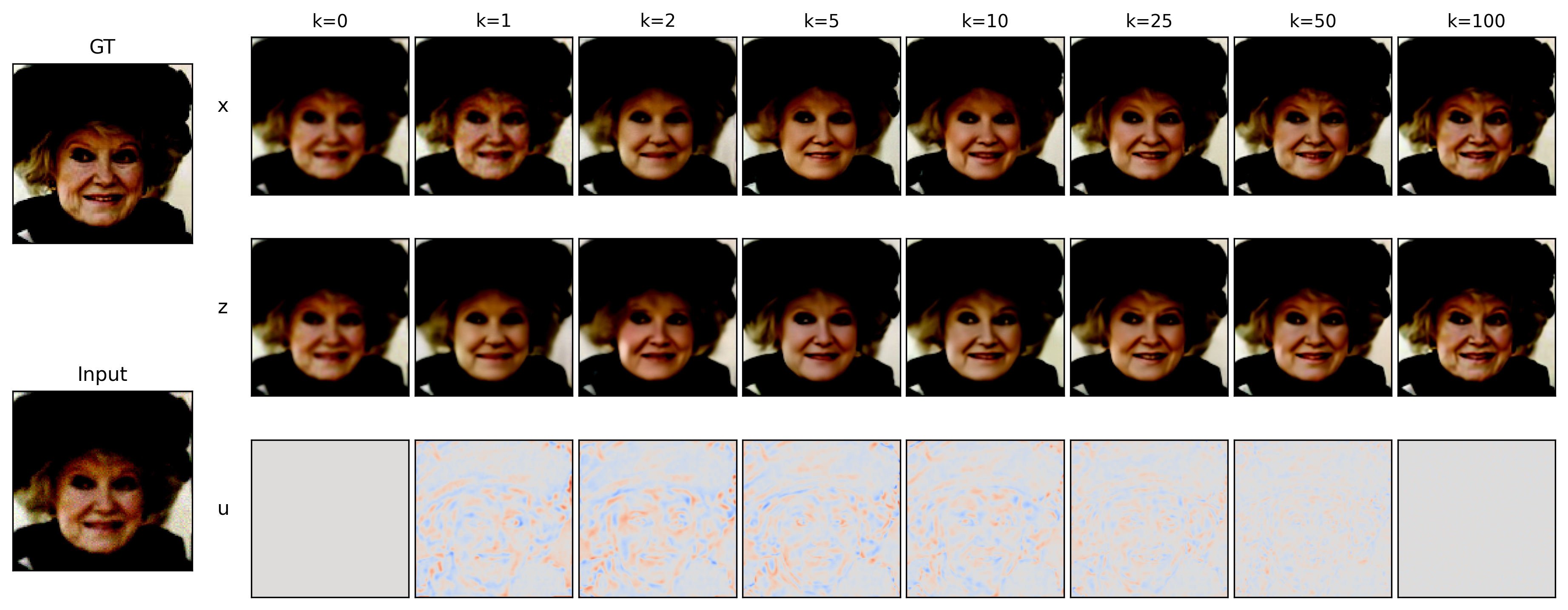}
    \caption{Trajectory visualization of a deblurring experiment}
    \label{fig:traj_deblurring}
\end{figure}

\begin{figure}[h]
    \centering
    \includegraphics[width=\linewidth]{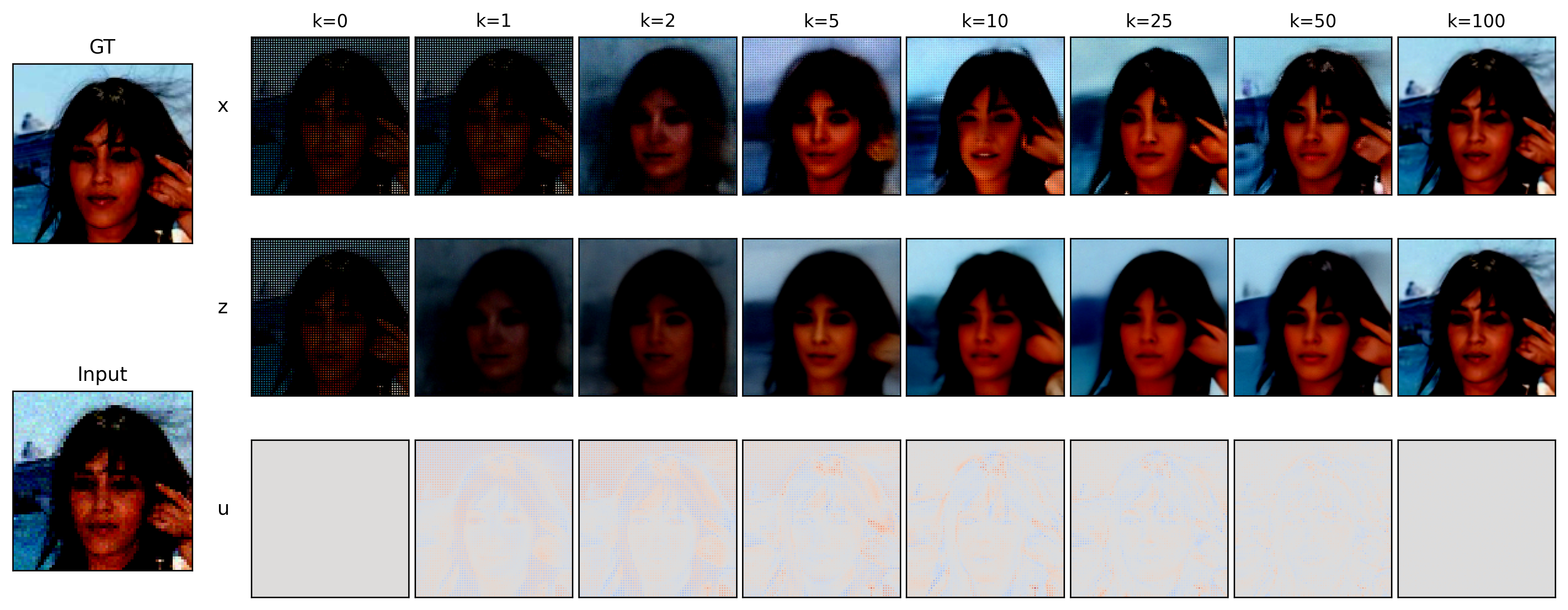}
    \caption{Trajectory visualization of a super-resolution experiment}
    \label{fig:traj_superres}
\end{figure}

\begin{figure}[h]
    \centering
    \includegraphics[width=\linewidth]{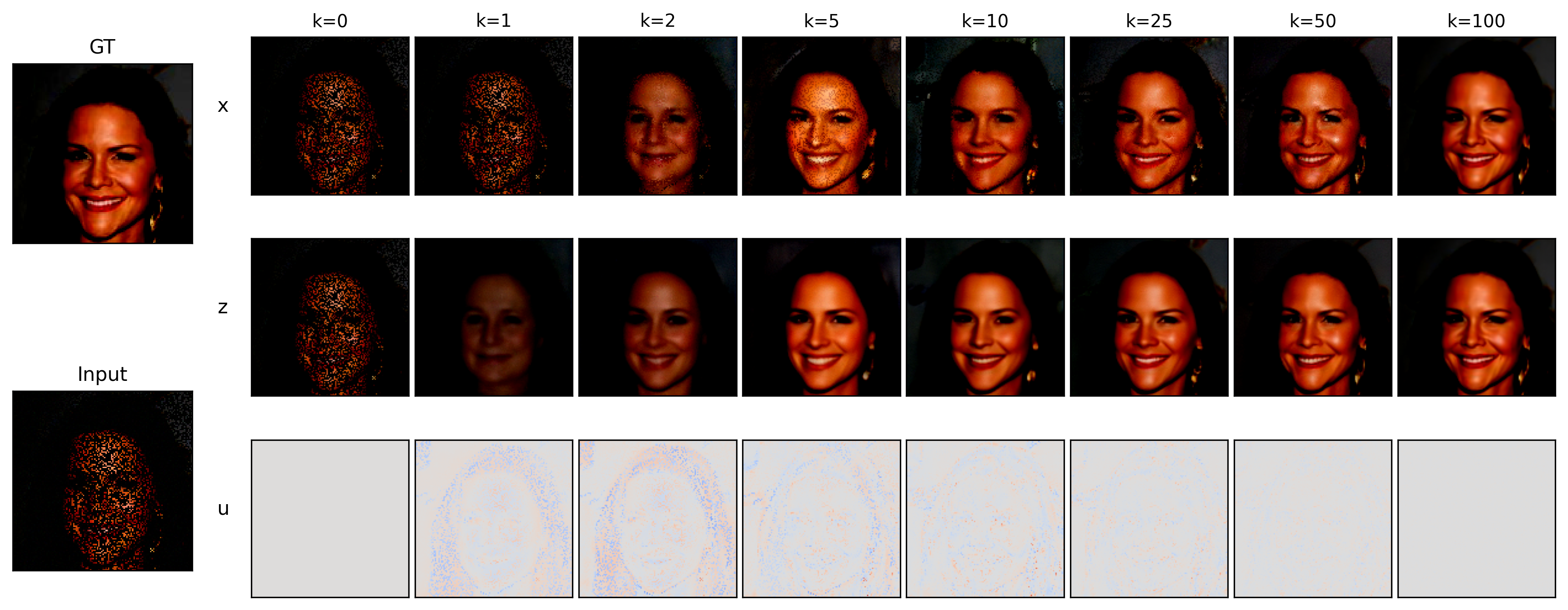}
    \caption{Trajectory visualization of a random inpainting experiment}
    \label{fig:traj_rand_inp}
\end{figure}

\begin{figure}[h]
    \centering
    \includegraphics[width=\linewidth]{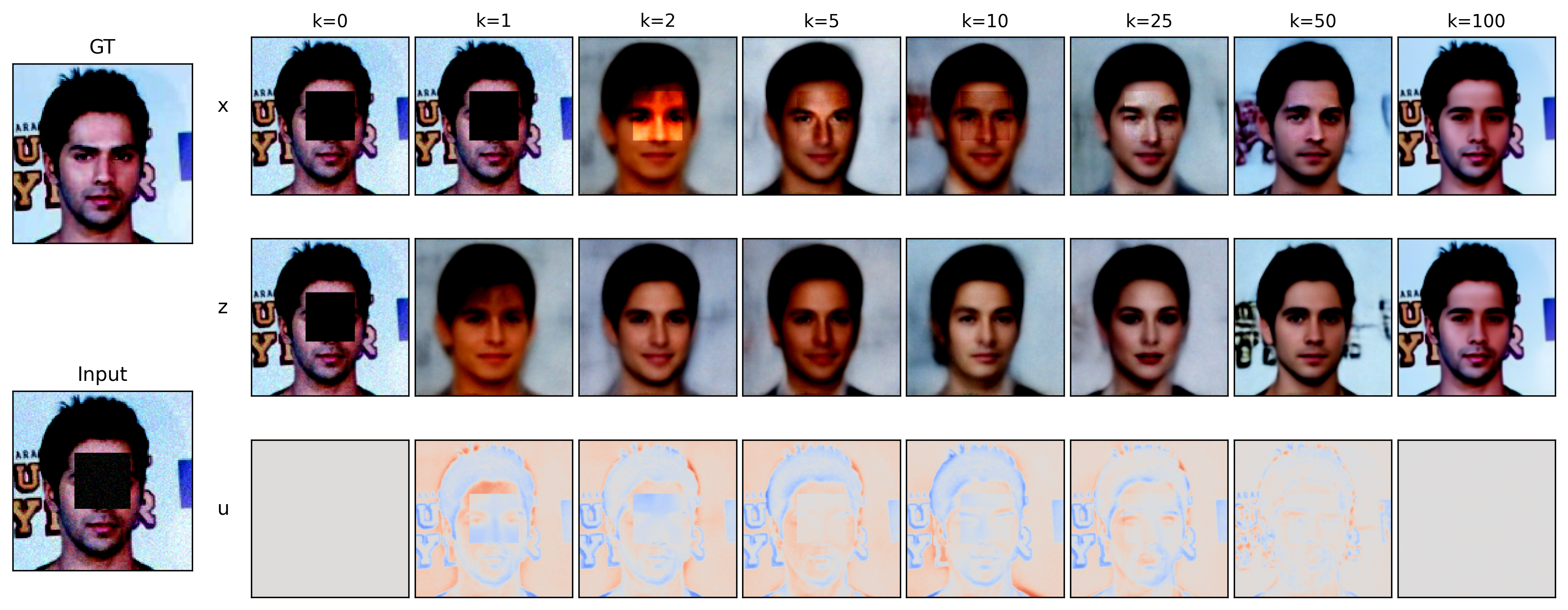}
    \caption{Trajectory visualization of a box inpainting experiment.}
    \label{fig:traj_box_inp}
\end{figure}

\subsection{Confidence intervals for improvements}
TO quanitify uncertainty in the reported test-set improvements over our main competitors PnP-Flow5 and Flower5-OT, we computed paired bootstrap confidence intervals over the per-image test PSNR.
For each task, we formed imagewise metric differences, over the same 100 test images, and then used bootstrapping to compute 95\% confidence intervals.
Several of the main gains of FlowADMM are clearly separated from zero, especially on our main improved tasks of super-resolution and random inpainting, which we report in Table~\ref{tab:bootstrap-ci}.
At the same time the intervals also confirm that the advantage is not uniform across all tasks: deblurring is the clearest exception, with FlowADMM underperforming Flower5-OT significantly.
Still, FlowADMM significantly outperforms or matches state-of-the art performance on most of the tasks.

\begin{table}[t]
  \centering
  \caption{Representative paired per-image bootstrap confidence intervals for
  PSNR differences. Positive values favor FlowADMM.}
  \label{tab:bootstrap-ci}
  \small
  \begin{tabular}{llc}
  \toprule
  Dataset & Task / Comparison & $\Delta$PSNR (95\% CI) \\
  \midrule
  CelebA & super-resolution vs. Flower5-OT & $+0.869\;[+0.818,+0.919]$ \\
  CelebA & random inpainting vs. Flower5-OT & $+1.172\;[+1.076,+1.270]$ \\
  AFHQ-Cat & super-resolution vs. Flower5-OT & $+1.138\;[+0.889,+1.411]$ \\
  AFHQ-Cat & deblurring vs. Flower5-OT & $-0.715\;[-0.959,-0.499]$ \\
  \bottomrule
  \end{tabular}
  \end{table}What 

\subsection{Additional tasks}    
We additionally consider the severely ill-posed problem of $8\times$ super-resolution on the CelebA dataset, which is not part of the standard benchmark used in prior work, which only evaluated $2\times$ super-resolution on that dataset.
We tuned FlowADMM and PnP-Flow on the validation set with the same procedure as usual. Flower does not expose hyperparameters for tuning, so we use it as is. As always, we keep the total number of flow iterations the same between methods.
The tuned hyperparameters for these experiments are shown in Table~\ref{tab:sr8-hparams}

The results are shown in Table~\ref{tab:sr8-stress}.
In this challenging setting, FlowADMM drastically outperforms the two competitors.

 \begin{table}[t]
  \centering
  \caption{Hyperparameters for the CelebA $8\times$ super-resolution test.}
  \label{tab:sr8-hparams}
  \small
  \setlength{\tabcolsep}{5pt}
  \begin{tabular}{ll}
  \toprule
  Method & Hyperparameters \\
  \midrule
  Flower5-OT
  & $(N,N_{\text{avg}})=(100,5)$ \\
  PnP-Flow5
  & $(N,N_{\text{avg}},\text{lr},\gamma\text{-style},\alpha)=(100,5,2.0,\text{const},0.001)$ \\
  FlowADMM
  & $(N,\tau,t_{\min},t_{\max},\gamma,N_k)$ \\
  & $=(100,0.1,0.2,0.95,1.0,\text{3ph}(1,3,35;0.6,0.9))$ \\
  \bottomrule
  \end{tabular}
  \end{table}

\begin{table}[t]
  \centering
  \caption{CelebA $8\times$ super-resolution on 100 test images. We compare against our main flow-based competitors}
  \label{tab:sr8-stress}
  \small
  \setlength{\tabcolsep}{5pt}
  \begin{tabular}{lccc}
  \toprule
  Method & PSNR $\uparrow$ & SSIM $\uparrow$ & LPIPS $\downarrow$ \\
  \midrule
  Flower5-OT & 14.21 & 0.374 & 0.383 \\
  PnP-Flow5 & 16.32 & 0.418 & 0.297 \\
  FlowADMM & \textbf{21.28} & \textbf{0.667} & \textbf{0.184} \\
  \bottomrule
  \end{tabular}
  \end{table}

\begin{figure}[h]
    \centering
    \includegraphics[width=\textwidth]{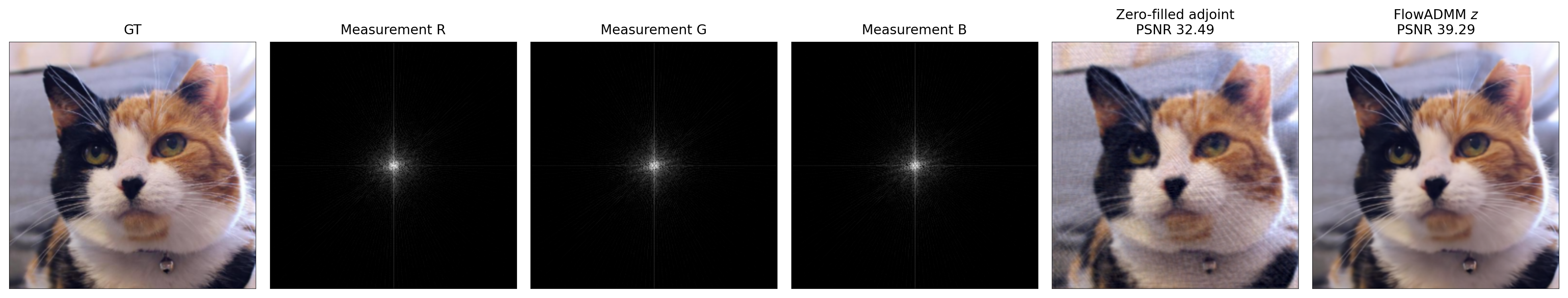}
    \caption{Results for CS MRI}
    \label{fig:cs-mri}
\end{figure}

\begin{figure}[h]
    \centering
    \includegraphics[width=\textwidth]{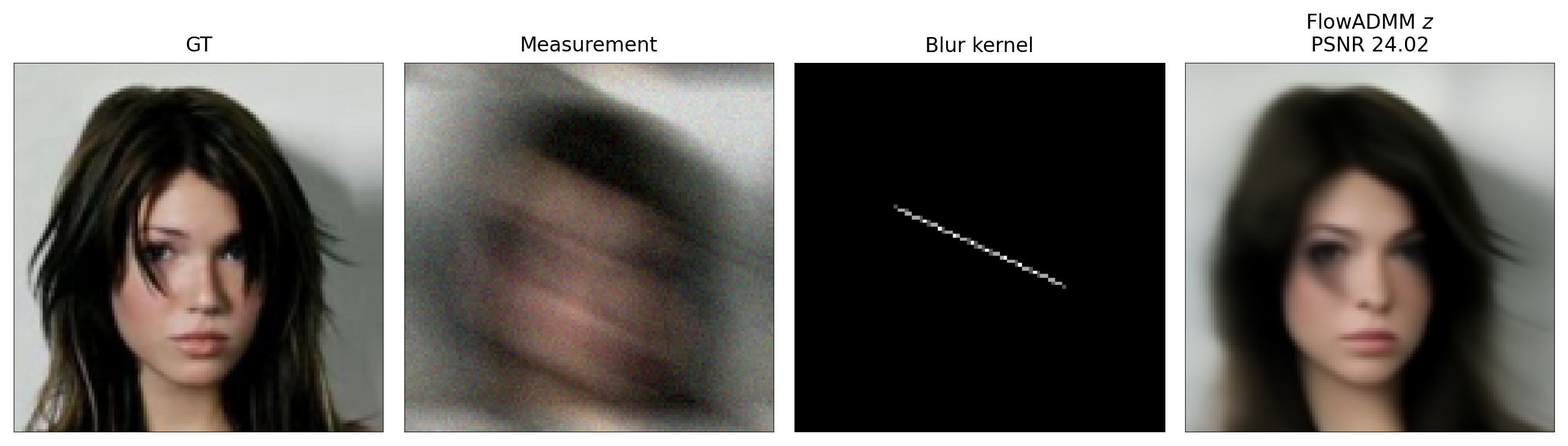}
    \caption{Results for motion deblurring}
    \label{fig:motion-deblur}
\end{figure}

We also tested the generality of our method on radial mask compressed sensing MRI (similar to the supplementary material of Flower) and motion deblurring. We did not tune hyperparameters for these tasks. The results can be found in Figures~\ref{fig:cs-mri} and \ref{fig:motion-deblur}.

%%%%%%%%%%%%%%%%%%%%%%%%%%%%%%%%%%%%%%%%%%%%%%%%%%%%%%%%%%%%

%\clearpage
%\input{checklist.tex}

\end{document}